\documentclass[nohyperref]{article}

\usepackage{microtype}
\usepackage{graphicx}
\usepackage{subfigure}
\usepackage{booktabs} 
\usepackage{float}
\usepackage{hyperref}
\usepackage[]{subfigure}

\usepackage[accepted]{icml2022}

\usepackage{amsmath}
\usepackage{amssymb}
\usepackage{mathtools}
\usepackage{amsthm}
\usepackage{blkarray}
\newcommand{\matindex}[1]{\mbox{\scriptsize#1}}

\usepackage[capitalize,noabbrev]{cleveref}

\theoremstyle{plain}
\newtheorem{theorem}{Theorem}[section]

\newtheorem{lemma}[theorem]{Lemma}
\newtheorem{corollary}[theorem]{Corollary}
\theoremstyle{definition}

\theoremstyle{remark}

\usepackage[textsize=tiny]{todonotes}

\icmltitlerunning{Lightweight Projective Derivative Codes}

\begin{document}

\twocolumn[
  \icmltitle{Lightweight Projective Derivative Codes \\ for Compressed Asynchronous Gradient Descent}



  \icmlsetsymbol{equal}{*}

  \begin{icmlauthorlist}
    \icmlauthor{Pedro Soto}{yyy}
    \icmlauthor{Ilia Ilmer}{yyy}
    \icmlauthor{Haibin Guan}{comp}
    \icmlauthor{Jun Li}{sch}
  \end{icmlauthorlist}

  \icmlaffiliation{yyy}{Department of Computer Science, The Graduate Center, CUNY, New York, USA}
  \icmlaffiliation{yyy}{Department of Computer Science, The Graduate Center, CUNY, New York, USA}
  \icmlaffiliation{comp}{Icahn School of Medicine at Mount Sinai, New York, USA}
  \icmlaffiliation{sch}{Department of Computer Science, CUNY Queens College \& Graduate Center, New York, USA}

  \icmlcorrespondingauthor{Pedro Soto}{psoto@gradcenter.cuny.edu}
  \icmlkeywords{Machine Learning, ICML}

  \vskip 0.3in
]



\printAffiliationsAndNotice{}  

\begin{abstract}
  Coded distributed computation has become common practice for performing gradient descent
  on large datasets to mitigate stragglers and other faults.
  This paper proposes a novel algorithm that \textit{encodes the partial derivatives} themselves and furthermore optimizes the codes by performing \textit{lossy compression} on the derivative codewords by maximizing the information contained in the codewords while minimizing the information between the codewords.
  The utility of this application of coding theory is a geometrical consequence of the observed fact in optimization research that noise is tolerable, sometimes even helpful, in gradient descent based learning algorithms since it helps avoid overfitting and local minima.
  This stands in contrast with much current conventional work on distributed coded computation which focuses on recovering all of the data from the workers.
  A second further contribution is that the \emph{low-weight} nature of the coding scheme allows for \emph{asynchronous gradient updates} since the code can be iteratively decoded;
  \emph{i.e.,} a worker's task can immediately be updated into the larger gradient.
  The directional derivative is always a linear function of the direction vectors; thus, our framework is robust since it can apply linear coding techniques to general machine learning frameworks such as deep neural networks.
\end{abstract}

\section{Introduction}

The majority of machine learning problems take the form: find a function
$
  h_{w_0,...,w_k}
$
in some family of hypothesis functions $\mathcal{H}$ that are parameterized over the $w_0,...,w_k$ which best explains the data,
\begin{equation}\label{eq:data_def}
  D = \begin{array}{c|cccccc}
              & x_0       & \dots     & x_u       &
    y_0       & \dots     & y_v                     \\
    \hline

    D_0       & x^{0}_1   & \dots     & x_u^{0}   &
    y_1^{0}   & \dots     & y_v^{0}                 \\

    D_1       & x^{1}_1   & \dots     & x^{1}_u   &
    y^{1}_1   & \dots     & y^{1}_v                 \\

    \vdots    & \vdots    & \ddots    & \vdots    &
    \vdots    & \ddots    & \vdots                  \\

    D_N       & x^{(N)}_1 & \dots     & x_u^{(N)} &
    y^{(N)}_1 & \dots     & y_v^{(N)}               \\
  \end{array},
\end{equation}
where the $D_i$ are the datapoints, the $x_i$ are the input features, and the $y_i$ are the output features.

If the $  h_{w_0,...,w_k}$ are smoothly parameterized by the $w_0,...,w_k$, then this is usually accomplished by performing gradient descent on the summation of loss functions of the form $l_i(w) = l(h_w(x^{(i)}), y^{(i)})$ to compute

\begin{align}
  \min_{w \in \mathcal{W}} \mathcal{L}(D,w) & \overset{\text{def}} {=} \min_{w \in \mathcal{W}} \sum_{D_i}l\left(h_w(x^{(i)}),y^{(i)}
  \right) =                                                               \nonumber                                                   \\\label{eq:min_loss}
                                            & =\min_{w \in \mathcal{W}} \sum_{D_i}l_i\left(w
  \right) ;
\end{align}
\textit{i.e., find the $w$ that best fits $D$.}
For example, one of the most common loss functions, $l(f,y)=\frac{1}{2}||f-y||^2$, gives us the mean squared error and the ubiquitous method of least squares.

If the dataset $D$ has many datapoints $D_i$ then the overall computation, or \textit{job}, is distributed as \textit{tasks} amongst \textit{workers}, which model a distributed network of computing devices. This solution creates a new problem; stragglers and other faults can severely impact the performance and overall training time. An emerging technique is to use distributed coded computation to mitigate stragglers and other failures in the network. Many of the current algorithms only encode the data; this paper proposes further encoding the directional derivatives as well in such a way that allows for asynchronous gradient updates using low weight codes.
Furthermore the number of weights usually grow quite large as well\footnote{As a matter of fact it grows proportionately with the number of features or \textit{dimension of the dataset}}, which necessitates a ``2D'' coding scheme which codes both the data and the derivatives.

\subsection{Related Work}
The two algorithms which we use to benchmark our algorithm are Gradient Coding \cite{pmlr-v70-tandon17a} and $K$-Asynchronous Gradient Descent \cite{pmlr-v84-dutta18a, Dutta2021SlowAS}; however, many of the design of our coding scheme is also influenced by the works in \cite{Lee2018a}, \cite{NIPS2017_e6c2dc3d}, and \cite{8765375}.

\subsubsection{Gradient Coding}
In the gradient coding Gradient Coding \cite{pmlr-v70-tandon17a} scheme the main idea is to encode the derivatives with respects to the data partitions $\frac{\partial}{\partial D_i}$ from Eq.~\eqref{eq:data_def};
since the loss function in Eq.~\eqref{eq:min_loss} splits up into a \emph{sum} of smaller loss functions, $l_i$, in terms of the partitions, $D_i$, linear codes can be efficiently applied to the gradients $\frac{\partial}{\partial D_i}$.
This work has gone on to spawn many works \cite{8635869, charles2018gradient, Ye2018CommunicationComputationEG, 8437467, 8849684,
  JMLR:v20:18-148, ozfatura2019gradient,  Ozfatura2019DistributedGD, 8849580, 9216021, 8849431, maity2019robust, 8989328, 8682911, NEURIPS2019_3eb2f1a0,
  9088154, 9368996, e22050544, Bitar2020StochasticGC, Zhang2021LAGCLA}; gradient coding is currently a vibrant topic of research.
The main improvements of our coding scheme over the state of the art in Gradient Coding is that our code: can perform asynchronous coded updates, allows the backpropagation itself to be coded (which greatly reduces the communication complexity for high dimensional data), our code has 0 encoding and decoding overhead in terms of multiplications, and has an overall reduction in the redundancy of data/memory overhead.

\subsubsection{Asynchronous Gradient Descent}
The main idea in Asynchronous Gradient Descent \cite{Ferdinand2017AnytimeEO, pmlr-v84-dutta18a, Ferdinand2018AnytimeSG,
  DBLP:journals/corr/abs-2006-05752, Dutta2021SlowAS} is to simply perform a gradient update when whenever a specified number, $k$, workers have returned.
The name ``asynchronous'' comes from the eponymous concept in distributed computing where communication rounds are not synchronized.

\subsubsection{General Coded Distributed Function Computation Schemes}
The main idea in \cite{Lee2018a}, \cite{NIPS2017_e6c2dc3d}, and \cite{8765375} is that one can distribute large matrix multiplications amongst workers and encode the smaller block matrix operations.
The works initiated much research in distributed coded matrix multiplication \cite{8765375, 8006963, 8437549, 8437852, pmlr-v80-wang18e, pmlr-v97-soto19a, 8758338, 8849395, hong21b}.
Further work has been extended to include batch matrix multiplication as well \cite{pmlr-v89-yu19b, 9149322, 9174239, 9750133}.
The main drawback of these (multi-)linear methods is the non-linear activation functions; in particular, these methods can only encode the linear computations between the layers of a network.
Another interesting approach is to attempt to encode the neural network itself \cite{DBLP:journals/corr/abs-1806-01259, 10.1145/3341301.3359654, Kosaian2019ParityMA}; however, this approach suffers from long training times dues to combinatorial explosion of different fault patterns is there are enough stragglers.

\subsection{Contribution}
The main contributions of this paper are to introduce a novel coding scheme for gradient descent that: allows for asynchronous gradient updates, maximizes the amount of information contained by random subsets of vectors, minimizes the weight of the code, compresses the gradient in a manner that scales well with the number of nodes, and achieves a lower a communication complexity and memory (storage) overhead with respect to the state of the art.
Another improvement of our algorithm over the state of the art is to consider the correct information metric; all of the other coding schemes assume that the Hamming distance is the correct metric, which does not consider the natural (differential) geometry of the gradient.
We will show that the correct distance is the one given by the \textit{real projective space}\footnote{$\mathbb{R P}^n$ is defined as the set of all vectors in $\mathbb{R }^{n+1}$ quotiented by the equivalence relation $v \sim w \iff (\exists \lambda)\ v = \lambda w$.} $\mathbb{R P}^n$.
Furthermore, we will show that our coding scheme, \emph{i.e.,} our choice of coefficients, maximizes the amount of information returned by the workers and furthermore has zero decoding overhead (in terms of multiplications) since the master can just directly add and subtract the results returned by the workers without needing to decode the information.

\subsection{Background}\label{subsec:back}
We quickly give some important definitions and background from coding theory, information theory, and geometry.
In coded distributed computing an \emph{erasure code} is a pair of functions $\mathcal{C} = (\mathcal{E}, \mathcal{D})$ where the workers tasks are given by the encoding procedure
\begin{equation*}
  \left\{ \tilde \theta_0 , ..., \tilde \theta_{n-1}\right\} :=  \mathcal{E} \left\{ \theta_0 , ...,  \theta_{k}\right\}
\end{equation*}
and a decoding procedure for some family of fault-tolerant subsets, $\mathcal{F}_\mathcal{C}$, such that
\begin{equation*}
  \left\{ \tilde \theta_{i_1} , ..., \tilde \theta_{i_m}\right\} \in \mathcal{F}_\mathcal{C} \implies \mathcal{D}\left\{ \tilde \theta_{i_1} , ..., \tilde \theta_{i_m}\right\} = \left\{ \theta_0 , ...,\theta_{k}\right\}.
\end{equation*}
If $\mathcal{F}_\mathcal{C}$ consists of all the $m$-subsets (for some integer $r$) of  $\left\{ \tilde \theta_0 , ..., \tilde \theta_{n-1}\right\}$, then $\mathcal{C}$ can correct any $r:= n-m$ erasures or stragglers; furthermore, if $r = n-k$ then the code is a \emph{maximum distance separable} (MDS) code.
If the encoder $\mathcal{E}$ is given by a generator matrix $\mathcal{G}_\mathcal{C}$, \emph{i.e.,} if
\begin{equation*}
  \mathcal{E} \begin{bmatrix}
    \theta_0 & ... & \theta_{k}
  \end{bmatrix}^T  =
  \mathcal{G}_\mathcal{C}\begin{bmatrix}
    \theta_0 & ... & \theta_{k}
  \end{bmatrix}^T
\end{equation*}
then $\mathcal{C}$ is called a \emph{linear code}. The \emph{weight} of a linear code is the maximum number of 0's in the rows of the matrix $\mathcal{G}_\mathcal{C}$; the importance of the weight metric stems from the fact that it measures the amount of work that the workers do since the rows of $\mathcal{G}_\mathcal{C}$ are the worker tasks $\tilde \theta_i$.
Thus, in order to avoid confusion we will use $t$ to denote the weight of the code as well as the number of \emph{tasks} that each worker does; equivalently $t$ is the number of data partitions on the workers.
To further simplify notation we abuse notation and use $\mathcal{C}$ in place of $\mathcal{G}_\mathcal{C}$ and $\mathcal{E}_\mathcal{C}$ when the context is clear.

A potential point of confusion is that the $\theta$ \emph{need not be} the weights $w$ of the $h_w$. This is because \emph{the derivative of loss function $\mathcal{L}$ also implicitly takes the data $D_i$ as an input;} this is an important insight used in all gradient coding algorithms. One of the key insights of this paper is to allow the coded gradient to be linear combinations of \textit{both} $\frac{\partial}{\partial D_i}$ \textit{and} the $\frac{\partial}{\partial w_i }$.
An important notational convention is that we let the $D_i$ be partitions (or batches) of the data set instead of just datapoints as is common in the gradient coding literature; in particular, $D_0,...,D_{t-1}$ denotes a partitioning of the data-set into $t$ pieces.

The reason for the name ``maximum distance separable code'' is that an MDS maximizes the distances between the codewords $\mathcal{E} \left\{ \theta_0 , ...,  \theta_{k}\right\}$ using the Hamming distance;
in particular, maximum distance separable means that the code words $\tilde \theta \in \mathcal{C}$ have achieve the maximum $\max _{\mathcal{C}': \text{code on }\Theta} \min_{\tilde \theta , \tilde \theta ' \in \mathcal{C}'}d(\tilde \theta , \tilde \theta ' )$ where $d$ is the Hamming distance.
There are two problems with this approach: the first is that MDS codes in this context require arbitrarily large amount of work, \emph{i.e.,} they have a large weight, and the second is that the classical \emph{discrete} MDS codes are using the wrong metric. This paper proposes to use the metric given by the projective geometry\footnote{See \cite{kuhnel2006differential} for the case $\mathbb{RP}^2$ and Appx.~3 of \cite{vogtmann2013mathematical} or Thm.~10.2 in ch.~3 of \cite{Suetin1989LinearAA} for the more general case $\mathbb{CP}^n$.} on the space of derivatives.
Here we mean maximum distance separable with respect to the distance function $d(\theta , \theta') =  \min\{\arccos \langle  \tilde \theta ,
  \tilde \theta '\rangle , \arccos \langle - \tilde \theta ,  \tilde \theta \rangle\}$.

\section{General Overview of the Design Principles}

Consider the case where there are two derivatives and we wish to create two parity tasks using only summation and subtraction in the encoding procedure.
Such a code is given by the following generator matrix
\begin{equation*}
  \mathcal{C}=\begin{blockarray}{cc}
    \ &  \matindex{$  \Theta_0 $}  \\
    \begin{block}{c[c]}
      \matindex{$\tilde \Theta_0$}  & I \\
      \matindex{$ \tilde \Theta_1$}  &  P  \\
    \end{block}
  \end{blockarray}=
  \begin{blockarray}{ccc}
    \  &  \matindex{$\theta_{0}$} & \matindex{$  \theta_{1}$} \\
    \begin{block}{c[cc]}
      \matindex{$\tilde \theta_{0}$}  &  1                 &  0  \\
      \matindex{$\tilde \theta_{1}$}  &  0                 & 1  \\
      \matindex{$\tilde \theta_{2}$}  & \frac{1}{\sqrt{2}} &  \frac{1}{\sqrt{2}} \\
      \matindex{$\tilde \theta_{3}$}  & \frac{1}{\sqrt{2}} &- \frac{1}{\sqrt{2}} \\
    \end{block}
  \end{blockarray}
\end{equation*}
which adds fault tolerance to the job
$
  I = \begin{bmatrix}
    1 & 0 \\
    0 & 1 \\
  \end{bmatrix}
$
with the \textit{parity} tasks
$
  P = \begin{bmatrix}
    \frac{1}{\sqrt{2}} & \frac{1}{\sqrt{2}}   \\
    \frac{1}{\sqrt{2}} & - \frac{1}{\sqrt{2}} \\
  \end{bmatrix}.
$
\textit{ This code has the serendipitous property of having negligible decoding complexity and negligible communication complexity!} For example, if the master receives $\nabla$ in the direction $\tilde \theta_3=\frac{1}{\sqrt{2}}(\theta_0+\theta_1)$, then the master can decrease both $ \theta_0$ and $\theta_1$ by the value returned by $\mathcal{W}_2$, \emph{i.e.,} $\tilde \theta _ 2$, if the master receives $\nabla$ in the direction $\tilde \theta_2=\frac{1}{\sqrt{2}}(\theta_0-\theta_1)$, then the master can decrease $ \theta_0$ and increase $\theta_1$  by the value returned by $\mathcal{W}_3$, \emph{i.e.,} $\tilde \theta _ 3$. \textit{The master need only perform 2 additions/subtractions, and more generally (see Eq.~\ref{eq:code}) if there are $t$ ``sub-tasks'' the master only needs to perform $t$ additions/subtractions. The multiplication by $\frac{1}{\sqrt{2}}$ can be subsumed by the learning rate; thus, our code has zero multiplication overhead.} Furthermore, this information can be communicated using only one float, since the master knows which direction/worker the derivative was computed from.

\subsection{What is the Optimal Choice of Directions?}

Looking at Fig.~\ref{fig:1}, we see that the code $\mathcal{G}$ is
MDS in the sense that it maximizes the independence between the vectors. Equivalently
\footnote{\label{fn:1}Under the assumption that the data $D_i$ are i.i.d.
  See \cite{Cover2006ElementsOI} for why maximal entropy maximizes information sent through a message.
}
$\mathcal{G}$ minimizes the confusion between codewords or minimizes the mutual information between codewords; thereby maximizing the entropy or the information content. As we will soon see this has the effect of allowing \textit{lossy low-distortion compression} for larger codes. A second contribution of this paper is to show how to preserve an \textit{approximate} MDS property for larger codes which allows for this form of compression.

In what sense does $\mathcal{C}$ being MDS imply fault tolerance? The following example illustrates one kind of error which the code is immune to:
\begin{figure}
  \centering
  \includegraphics[width=.48\textwidth]{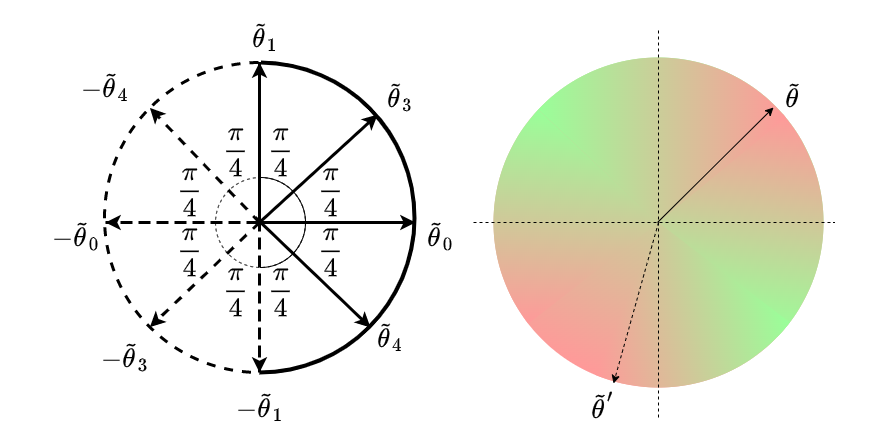}
  \caption{The code $\mathcal{C}$ is MDS with respect to the the distance between the codewords.
    It is easy to see that $- \tilde \theta_i$ carries the same information as $ \tilde \theta_i$ and it is therefore inefficient to include it since it is a form of replication coding. Likewise we should not include a vector that makes a small angle with $ \tilde \theta_i$ for the same reason.
    The geometry of this problem is that of $\mathbb{RP}^ 1$, \textit{i.e., the real protective line.} This is because the directions $\tilde \theta$ and $-\tilde \theta$ are information theoretically equivalent. The second figure displays the relationship between $\tilde \theta$ and $\tilde \theta ' $. As $\tilde \theta ' $ becomes more linearly independent from $\tilde \theta$ it begins to carry more novel information that cannot be inferred from $\tilde \theta$. At the ``greenest'' extremes $\tilde \theta$ and $\tilde \theta ' $ become statistically independent which is the maximum entropy configuration.
    Maximum entropy is equivalent (see Fn.~\ref{fn:1}) to maximum information about the $\theta_i$.
  }\label{fig:1}
\end{figure}

Consider the case where two workers return the derivatives in the directions
$
  \tilde \theta  = \begin{pmatrix} \frac{1}{\sqrt{2}} & \frac{1}{\sqrt{2}}\end{pmatrix}^T, \tilde \theta ' = \begin{pmatrix} \cos\left(\frac{5\pi}{4}- \epsilon\right) & \sin\left(\frac{5\pi}{4}- \epsilon\right)\end{pmatrix}^T.
$
By inspecting the second diagram in Fig.~\ref{fig:1}, it is easy to see that
$
  \lim_{\epsilon \to 0}  \tilde \theta '  \to  - \tilde \theta,
$
so that
$
  \lim_{\epsilon \to 0}  \arccos \left\langle  \tilde \theta ,  \tilde \theta '\right\rangle \to  \pi .
$
We will also show that the error in the derivative can get bigger and bigger as $\epsilon \to 0 $.

If worker one computes $  \frac{\partial \mathcal{L}}{\partial \tilde \theta}$, worker two computes $  \frac{\partial \mathcal{L}}{\partial \tilde \theta'}$, and
$
  \frac{\partial l}{\partial \theta_i}   = (-1)^{i}L,
$
then we have that
$
  \frac{\partial \mathcal{L}}{\partial \tilde \theta} = \frac{1}{\sqrt{2}}(L+(-L))=0
$
and that
$
  \frac{\partial \mathcal{L}}{\partial \tilde \theta'}= \cos\left(\frac{5\pi}{4}- \epsilon\right)L-\sin\left(\frac{5\pi}{4}- \epsilon\right)L= \left(\cos\left(\frac{5\pi}{4}- \epsilon\right)-\sin\left(\frac{5\pi}{4}- \epsilon\right)\right)L.
$
Therefore if both
$
  ( \epsilon \approx 0) $ and $( L >>1)
$,
then
$
  \frac{\partial \mathcal{L}}{\partial \tilde \theta} \approx\frac{\partial \mathcal{L}}{\partial \tilde \theta'}\approx 0 ,
$
which is an error; \textit{when the master receives the messages from workers one and two she will think she has arrived at a optimal fit since both $\frac{\partial \mathcal{L}}{\partial \tilde \theta} $ and $\frac{\partial \mathcal{L}}{\partial \tilde \theta'} $ are very small; \emph{i.e.,} the master may halt the algorithm on a terrible fit.} Furthermore it is easy to see that
$
  \max_{\epsilon }\frac{\partial \mathcal{L}}{\partial \tilde \theta'}
$
occurs when $\epsilon= \frac{\pm \pi}{2}$, \emph{i.e.,} when
$
  \arccos \left\langle  \tilde \theta ,  \tilde \theta '\right\rangle = \frac{\pi}{2}
$
so that our previous choice is optimal.

\subsection{Motivating Example: Base Case of LWPD Codes}\label{subsec:mot_ex}

The last example did not allow us to show the more general \textit{lossy compression} phenomenon that can occur for more general codes. Also we will soon prove that it is impossible to have MDS codes for large dimensions where the workers perform a small amount of work\footnote{This follows from a general rule of thumb in coding theory which states that a code cannot simultaneously have a sparse matrix and be MDS; however we will prove it rigorously for our case.}. In a sense, $k = 2$ is a very special case. Therefore before showing the general compression phenomenon, let us show how to \textit{compress the derivative} for $k = 4$.

Suppose that we have 8 workers $\tilde \theta_i$, loss function $l(h,y)=\frac{1}{2}||h-y||^2$, $u$ input features $x_i$, and space of hypothesis functions
\begin{equation*}
  \mathcal{H} = \left\lbrace h_{w} = (y_1,...,y_v)  \middle| y_i = \frac{e^{w_i^Tx}}{1 + \sum_j^{w_j^Tx}}\ , \ w_i \in \mathbb{R}^2 \right\rbrace,
\end{equation*}
where $w_j^Tx = w_{j,1}x_1 + ... + w_{j,u}x_u$; \emph{i.e.,} $\mathcal{H}$ is the space of multinomial logistic regression functions (\emph{however, this procedure will work for any feed-forward deep neural network}, see Fig.~\ref{fig:weight_partition}).
Similarly to the previous design we can give the following directions to the workers
\begin{equation}\label{eq:basecode}
  \mathcal{C}^{(8,4,2)} =  \begin{blockarray}{ccccc}
    \ &  \matindex{$\theta_{0}$} & \matindex{$ \theta_{1}$} & \matindex{$ \theta_{2}$} & \matindex{$  \theta_{3}$} \\
    \begin{block}{c[cccc]}
      \matindex{$\tilde \theta_{0}$}
      & \frac{1}{\sqrt{2}} &  \frac{1}{\sqrt{2}}  & 0& 0 \\
      \matindex{$\tilde  \theta_{1}$}
      &  \frac{1}{\sqrt{2}} & -\frac{1}{\sqrt{2}} & 0 &0  \\
      \matindex{$\tilde  \theta_{2}$}
      &  0 & 0 & \frac{1}{\sqrt{2}} &\frac{1}{\sqrt{2}}\\
      \matindex{$\tilde  \theta_{3}$}
      &  0 & 0 & \frac{1}{\sqrt{2}} &-\frac{1}{\sqrt{2}}\\
      \matindex{$\tilde \theta_{4}$}
      & 0 & \frac{1}{\sqrt{2}} &  \frac{1}{\sqrt{2}}  & 0\\
      \matindex{$\tilde  \theta_{5}$}
      & 0  &  \frac{1}{\sqrt{2}} & -\frac{1}{\sqrt{2}} &0 \\
      \matindex{$\tilde  \theta_{6}$}
      &\frac{1}{\sqrt{2}} &  0 & 0 & \frac{1}{\sqrt{2}}  \\
      \matindex{$\tilde  \theta_{7}$}
      &-\frac{1}{\sqrt{2}} &  0 & 0 & \frac{1}{\sqrt{2}}   \\
    \end{block}
  \end{blockarray};
\end{equation}
however, this time we let $\frac{\partial}{\partial \theta_0} = $``the derivative of the first half of the output nodes with respect to the first half of the dataset'', $\frac{\partial}{\partial \theta_1} = $``the derivative of the second half of the output nodes with respect to the first half of the dataset'', $\frac{\partial}{\partial \theta_2} = $``the derivative of the first half of the output nodes with respect to the second half of the dataset'', and $\frac{\partial}{\partial \theta_3} = $``the derivative of the second half of the output nodes with respect to the second half of the dataset''.
We can see that ``lossy'' part of the lossy compression is that the 4 workers don't necessarily return the
gradient perfectly, but we will later prove that they return a pretty good approximation of it. However, we had a second further lossy compression step to our code; we give workers 5 and 6 the data partitions $D_2, D_3$ and give workers 7 and 8 the data partitions $D_1, D_4$ instead of giving these workers all of the partitions as the code in Eq.~\ref{eq:basecode} suggests.
\begin{figure}
  \centering
  \includegraphics[width=.41\textwidth]{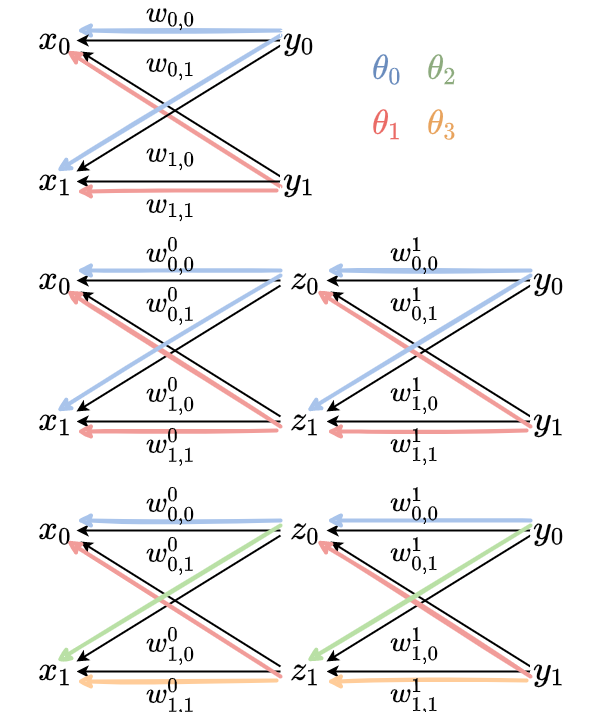}
  \caption{The different ways to partition the backpropagation gradient. The first partition shows how to partition the gradient for a simple neural network with no hidden nodes. The two other partition corresponds to a more general deep-neural network with hidden nodes. The last partition shows how to apply the recursive step in Alg.~\ref{alg:grad_part}.}\label{fig:weight_partition}
\end{figure}

\section{General Code Construction}

We first show how to construct what we will denote as a \textit{$[n,k,t]$-projective derivative code}, or $[n,k,t]$-code, for $n=2^m$, $k = 2^l$ and $t=2^p$, \emph{i.e.,} $n,k,t$ are all powers of 2, and then show how to use ``cyclic'' and ``toroidal'' permutations to construct the code for more general $k,n$; however, the $t$ is always chosen to be a power of two for reasons that will soon become clear. The parameter $n$ is the number of workers, $k$ the number of derivatives/features, and $t$ is the number of sub-tasks that each worker will perform; \emph{i.e.,} the number of derivatives per worker.

\subsection{The Characteristic Vectors}
To construct the code we first construct \textit{the characteristic vectors} from the following family of functions
$
  \chi_{\alpha} : \mathbb{F}^p_2 \longrightarrow \mathbb{C},
$
defined by the lambda expression
\begin{equation*}
  \chi_{\alpha} : \beta \mapsto  \frac{1}{\sqrt{2^p}}e^{ \left\langle \alpha , \beta  \right\rangle \pi i}  = \frac{(-1)^{ \left\langle \alpha , \beta  \right\rangle }}{\sqrt{2^p}},
\end{equation*}
where $\alpha, \beta \in \mathbb{F}^p_2$ are defined as binary strings of length $t$ and $  \left\langle \alpha , \beta  \right\rangle $ is the \textit{dot product} on $\alpha, \beta \in \mathbb{F}^p_2$, which is equivalent to taking the \textit{bit-wise }AND\footnote{AND is the \textit{logical conjunction}, denoted by $\land$,
  \textit{i.e.,} $a \land b$ = 1 if $a = b = 1$ and $a \land b$ = 0 otherwise. The bit-wise AND of two sequences $\alpha = \alpha_0...\alpha_{p-1},\beta = \beta_0...\beta_{p-1}$ is $\alpha \land \beta  = (\alpha_0 \land \beta _0)(\alpha_1 \land \beta _1) ...(\alpha_{p-1} \land \beta_{p-1}) $} of $\alpha$ and $\beta$ and then taking the XOR\footnote{\label{fn:2}XOR is the \textit{exclusive or}, denoted by $\oplus$,
  \textit{i.e.,} $a \oplus b$ = 1 if $a \neq b$ and $a \oplus b$ = 0 if $a = b$.} of the result. In particular, $  \left\langle \alpha , \beta  \right\rangle $ is defined as
$
  \left\langle \alpha , \beta  \right\rangle  = (\alpha _0 \land \beta _0) \oplus  (\alpha_1 \land \beta _1) \oplus \dots\oplus   (\alpha _{p-1} \land \beta _{p-1}),
$ where $\oplus$ is as defined in fn.~\ref{fn:2}.

It is an elementary fact of representation theory that the vectors, $\sqrt{2^p}\chi_\alpha$ correspond to the \textit{irreducible representations} of $\mathbb{F}^p_2$ in $\mathbb{C}$ and are therefore an orthogonal basis see Thm.~6 of \cite{scott2012linear} or Thm.~2.12 of \cite{fulton1991representation}. One can also prove this fact by direct computation using discrete Fourier analysis see ch.~4 of \cite{tao2006additive}. These functions are well-studied in discrete mathematics and usually referred to as the (additive) characters of $\mathbb{F}^p_2$.

Let us construct these vectors for $p= 2$ and verify the veracity of these statements for that case. The binary strings of length $2$ are $ \alpha \in \{00,01,10,11\} $ and this corresponds to the functions


\begin{equation*}
  \resizebox{\hsize}{!}{$
      X^{4}
      =  \!\!\!  \begin{blockarray}{ccccc}
        \ &  \matindex{$\theta_0$} & \matindex{$\theta_1$} & \matindex{$\theta_2$} & \matindex{$\theta_3$}   \\
        \begin{block}{c@{\hspace{5.6pt}}[c@{\hspace{5.6pt}}c@{\hspace{5.6pt}}c@{\hspace{5.6pt}}c]}
          \matindex{$\chi_{00}$}
          & (-2)^{\frac{-p}{2} \left\langle 00 , 00  \right\rangle }
          & (-2)^{\frac{-p}{2} \left\langle 00 , 01  \right\rangle }
          & (-2)^{\frac{-p}{2} \left\langle 00 , 10  \right\rangle }
          & (-2)^{\frac{-p}{2} \left\langle 00 , 11  \right\rangle }\\
          \matindex{$\chi_{01}$}
          & (-2)^{\frac{-p}{2} \left\langle 01 , 00  \right\rangle }
          & (-2)^{\frac{-p}{2} \left\langle 01 , 01  \right\rangle }
          & (-2)^{\frac{-p}{2} \left\langle 01 , 10  \right\rangle }
          & (-2)^{\frac{-p}{2} \left\langle 01 , 11  \right\rangle }\\
          \matindex{$\chi_{10}$}
          & (-2)^{\frac{-p}{2} \left\langle 10 , 00  \right\rangle }
          & (-2)^{\frac{-p}{2} \left\langle 10 , 01  \right\rangle }
          & (-2)^{\frac{-p}{2} \left\langle 10 , 10  \right\rangle }
          & (-2)^{\frac{-p}{2} \left\langle 10 , 11  \right\rangle }\\
          \matindex{$\chi_{11}$}
          & (-2)^{\frac{-p}{2} \left\langle 11 , 00  \right\rangle }
          & (-2)^{\frac{-p}{2} \left\langle 11 , 01  \right\rangle }
          & (-2)^{\frac{-p}{2} \left\langle 11 , 10  \right\rangle }
          & (-2)^{\frac{-p}{2} \left\langle 11 , 11  \right\rangle }\\
        \end{block}
      \end{blockarray}\!
    $}
\end{equation*}
\begin{equation*}
  =\!\!\! \begin{blockarray}{ccccc}
    \ &  \matindex{$\theta_0$} & \matindex{$\theta_1$} & \matindex{$\theta_2$} & \matindex{$\theta_3$}   \\
    \begin{block}{c@{\hspace{5.6pt}}[c@{\hspace{5.6pt}}c@{\hspace{5.6pt}}c@{\hspace{5.6pt}}c]}
      \matindex{$\chi_{00}$}
      & 1/2 & 1/2  & 1/2 &1/2\\
      \matindex{$\chi_{01}$}
      & 1/2 &   -1/2 & 1/2 & -1/2\\
      \matindex{$\chi_{10}$}
      & 1/2 & 1/2  & -1/2 & -1/2 \\
      \matindex{$\chi_{11}$}
      & 1/2 &  -1/2 & -1/2 & 1/2  \\
    \end{block}
  \end{blockarray},
\end{equation*}

and it is straightforward to see that all of the vectors $\chi_{00}$, $\chi_{01}$, $\chi_{10}$, and $\chi_{11}$ are orthonormal.

\subsection{Construction for Powers of Two}\label{subsec:pow_two_con}
If we identify the binary strings $\alpha,\beta$ with the integers that they represent and let  $X^{(t)}$ be the matrix defined coordinate-wise by the equation
\begin{equation*}
  X^{(t)}_{\alpha,\beta} = \chi_\alpha(\beta) =  \frac{1}{\sqrt{t}}e^{ \left\langle \alpha , \beta  \right\rangle \pi i},
\end{equation*}
where $p = \log(t)$ and $ \chi_{\alpha} : \mathbb{F}^p_2 \longrightarrow \mathbb{C}$ and let $L^{(t)}$ and  $R^{(t)}$ be the matrices defined by the equation
\begin{equation*}
  L^{(2t)} = \begin{bmatrix}
    0^{\left(t \right)} & \frac{1}{\sqrt{2}}X^{\left( t \right)} \\
    0^{\left(t \right)} & \frac{1}{\sqrt{2}}X^{\left( t \right)} \\
  \end{bmatrix}, \ \ \
  R^{(2t)}  =
  \begin{bmatrix}
    \frac{1}{\sqrt{2}}X^{\left( t \right)}   & 0^{\left(t \right)} \\
    - \frac{1}{\sqrt{2}}X^{\left( t \right)} & 0^{\left(t \right)} \\
  \end{bmatrix},
\end{equation*}
then we can define $\mathcal{C}^{(2k,k,t)} $, the generator for the $[2k,k,t]$-code, as
\begin{equation}\label{eq:code}
  \begin{blockarray}{ccccccc}
    \ &   \matindex{$ \theta^{(t)}_0$}
    &  \matindex{$\theta^{(t)}_1$}
    & \matindex{$\theta^{(t)}_2$}
    & \dots
    &  \matindex{$\theta^{(t)}_{s-2}$}
    & \matindex{$\theta^{(t)}_{s-1}$}\\
    \begin{block}{c[cccccc]}
      \matindex{$\tilde  \theta^{(t)}_0$}
      & X^{(t)} &  0       & 0       & \hdots &  0    & 0      \\
      \matindex{$\tilde  \theta^{(t)}_1$}
      &  0      & X^{(t)}  & 0       & \hdots & 0     & 0      \\
      \matindex{$\tilde \theta^{(t)}_2$}
      & 0       &  0       & X^{(t)} & \hdots &\hdots &0       \\
      \matindex{$ \vdots$}
      &  \vdots & \vdots   &\vdots   &\ddots  & \vdots&\vdots  \\
      \matindex{$ \tilde \theta^{(t)}_{s-2} $}
      &  0      & 0        & 0       &\hdots  &X^{(t)}&0       \\
      \matindex{$ \tilde \theta^{(t)}_{s-1}$}
      &  0      & 0        & 0       &\hdots  &0      &X^{(t)} \\
      \matindex{$\tilde \theta^{(t)}_{s}$}
      & L^{(t)} &  R^{(t)} &  0      & \hdots & 0     &0       \\
      \matindex{$\tilde  \theta^{(t)}_{s+1}$}
      &  0      &  L^{(t)} & R^{(t)} & \hdots &0      &0       \\
      \matindex{$ \vdots$}
      &  \vdots & \vdots   &\vdots   &\ddots  & \vdots& \vdots \\
      \matindex{$ \tilde\theta^{(t)}_{2s-2} $}
      &  0      & 0        & 0       &\hdots  &L^{(t)}&R^{(t)} \\
      \matindex{$\tilde \theta^{(t)}_{2s-1}$}
      & R^{(t)} & 0        & 0       &\hdots  &0      &L^{(t)} \\
    \end{block}
  \end{blockarray},
\end{equation}
where\footnote{Equivalently, we can write these definitions as $s = \frac{k}{t}$ and $ \theta^{(t)}_{i} =   \theta_{it}\theta_{it+1} ... \theta_{(i+1)t-1} $} $s$ is the ratio of tasks to sub-tasks, $ \theta^{(t)}_{i} $ is the sequence of sub-tasks $\theta_{it}$ through $\theta_{(i+1)t-1}$, and $ \tilde \theta^{(t)}_{i}$ is  similarly defined as a sequence of $t$ consecutive workers. Equivalently if we define the ```rectangles'' $\mathcal{R}^{(t)}_{u,v}$ as
\begin{equation*}
  \mathcal{R}_{u,v}^{(t)} = \{(i,j) \in \mathbb{N}^2 \ | \  ut \leq i < (u+1)t ,  \  vt \leq j < (v+1)t  \},
\end{equation*}
then we can define $\mathcal{C}^{(2k,k,t)} $ coordinate -wise as
\begin{equation*}
  \resizebox{\hsize}{!}{$
      \mathcal{C}^{(2k,k,t)}_{\theta _i , \tilde \theta _j} =
      \begin{cases}
        X^{(t)}_{i \% t,j\% t } & \text{if } i < k \text{ and } ( i,j)  \in \mathcal{R}_{\left\lfloor\frac{i}{t} \right\rfloor,\left\lfloor\frac{i}{t} \right\rfloor}^{(t)}               \\
        L^{(t)}_{i \% t,j\% t } & \text{if } k\leq i \text{ and }( i,j)  \in \mathcal{R}_{\left\lfloor\frac{i}{t} \right\rfloor+\frac{k}{t},\left\lfloor\frac{i}{t} \right\rfloor}^{(t)}  \\
        R^{(t)}_{i \% t,j\% t } & \text{if }( i,j)  \in \mathcal{R}_{\frac{2k}{t}-1,0}^{(t)} \text{ or } k\leq i                                                                          \\
                                & \text{ and }t\leq j \text{ and }( i,j)  \in \mathcal{R}_{\left\lfloor\frac{i}{t} \right\rfloor+\frac{k}{t},\left\lfloor\frac{i}{t} \right\rfloor}^{(t)} \\

        0                       & \text{otherwise}.                                                                                                                                       \\
      \end{cases}
    $}
\end{equation*}

It is straightforward to prove the following beautiful property
\begin{lemma}\normalfont\label{lem:tensor}
  The matrices $X^{(t)}$ satisfy the following recursion relation $X^{(2t)} = X^{(2)} \otimes X^{(t)}  $.
\end{lemma}
\begin{proof}
  This is a direct consequence of Thm.~10 in \cite{scott2012linear}.
\end{proof}
An alternative is the weaker statement\footnote{Although a weaker statement it suffices to to prove the claim of optimally,\emph{i.e.,} Thm.~\ref{thm:opt}.} ``$X^{(2)}$ is a Hadamard matrix and the tensor product of two Hadamard matrices is a Hadamard matrix'' whose proof can be found in \cite{2003fundamentals}.

\subsubsection{Data-\&-Gradient-Partition for Powers of Two}
Similar to the example given in Sec.~\ref{subsec:mot_ex} we give the workers $\tilde \theta_i $ the data partition given by Alg.~\ref{alg:data_part}.
\begin{algorithm}
  \caption{Data\_Partition\_Assignment}
  \label{alg:data_part}
  \begin{algorithmic}
    \STATE {\bfseries Input:} \texttt{data }$D$,
    \texttt{code\_parameters} $(n,k,t)$
    \STATE Partition the data $D$ into $D_0, .., D_{k-1}$
    \STATE Set $\mathcal{C} := \mathcal{C} ^{(n,k,t)}$
    \FOR{$i \leq n$}
    \STATE \texttt{Data}$[ \tilde \theta_i] :=\emptyset$
    \FOR{$j \leq k$}
    \IF{$\mathcal{C}_{i,j} \neq 0 $}
    \STATE Set \texttt{Data}$[ \tilde \theta_i]:=\texttt{Data}[\tilde \theta_i]\cup D_j$
    \ENDIF
    \ENDFOR
    \ENDFOR
  \end{algorithmic}
\end{algorithm}
The idea behind Alg.~\ref{alg:data_part} is simple; we give the first worker \texttt{Data}$[ \tilde \theta_0]= D_0,...,D_{t-1}$, and the second worker \texttt{Data}$[ \tilde \theta_1]= D_t,...,D_{2t-1}$, and so on up to worker $k$, at which point we give the workers $k,...,n$ a cyclic shift of the previous assignment, \emph{e.g,} worker $k$ gets \texttt{Data}$[ \tilde \theta_k]= D_{\frac{t}{2}},...,D_{t+\frac{t}{2}}$.

The procedure for partitioning and encoding the gradients, Alg.~\ref{alg:grad_part}, is slightly more involved; however, the main idea is illustrated in Fig.~\ref{fig:weight_partition}.
\begin{algorithm}
  \caption{Gradient\_Partition\_Assignment}
  \label{alg:grad_part}
  \begin{algorithmic}
    \STATE {\bfseries Input:} \texttt{network} $x,z^0,...,z^m,y$, \texttt{code\_parameters} $(n,k,t)$
    \STATE Set $\mathcal{C} := \mathcal{C} ^{(n,k,t)}$
    \STATE Partition $y$ into $t $ groups $y^{(i)}$
    \STATE where $y^{(0)}=(y_0,...,y_{t-1});\dots ; y^{(t)}=(y_{v-t},...,y_{v})$
    \FOR{$i \leq n$}
    \STATE Encode \texttt{grad}$[\tilde \theta_i]$ according to row $i$ in $\mathcal{C}$ as in Fig.~\ref{fig:weight_partition}
    \ENDFOR
    \IF{\texttt{network} == $x,y$}   \comment{Base Case}
    \STATE End Procedure
    \ELSE \comment{Induction Step}
    \FOR{$i \leq m$}
    \STATE Recursively call ``Gradient\_Partition\_Assignment'' on the network $x,z^0,...,z^m$ parameters $(n,k,t)$ as in Fig.~\ref{fig:weight_partition} to encode \texttt{grad}$[\tilde \theta_i]$ according to row $i$ in $\mathcal{C}$ by repeatedly splitting the (non-zero-)row by $t$
    \ENDFOR
    \ENDIF
  \end{algorithmic}
\end{algorithm}
The main intuition behind Alg.~\ref{alg:grad_part} is to encode the gradient in the \emph{manner in which backpropagation occurs;} this allows for the iterative decoding/gradient update at the master node, which in turn allows for asynchronous gradient updating.
\subsection{Construction for General Parameters}

Given some general $(n,k,t) $ we construct the matrix $\mathcal{C}^{(n',k',t)} $, where $n'$ and $k'$ are the next nearest powers of 2 (repeating rows if necessary) and use a ``2-D'' permutation algorithm similar to \cite{9213028} to distribute the sub-tasks in each round; however our algorithm uses more general (prime number) step-sizes chosen in each round and the permutations now occur in ``higher dimensions\footnote{\emph{I.e.,} our algorithm permutes more than one index; in particular, it permutes the subtask, worker, and output indices in order to create ``3-D'' permutations. The step-sizes are more general because they must be co-prime to one another to ensure every blue rectangle (see Eq.~\ref{eq:spin}) is visited.}.''
In particular; we now use a similar procedure to permute tasks amongst workers if $n$ and $k$ are not powers of 2.
For example if we have $n=6$ workers and $k=3$ tasks we can add extra virtual tasks $\theta_3 = \theta_0,$ $\theta_4 = \theta_1,$ $\dots,$ $\theta_{x} = \theta_{x\%3}$ and perform the following toroidal permutations on  $\mathcal{C}^{(5,3,2)}$
\begin{figure*}
  \begin{equation}\label{eq:spin}\begin{gathered}
      \begin{blockarray}{cccccccccc}
        \ & &  \matindex{$\tilde \vartheta_{0}$} & \matindex{$ \tilde \vartheta_{1}$} & \matindex{$\tilde \vartheta_{2}$} & \matindex{$\tilde \vartheta_{3}$} &  \matindex{$\tilde \vartheta_{4}$} &  & & \\
        \ & &  \matindex{$\tilde \theta_{0}$} & \matindex{$ \tilde \theta_{1}$} & \matindex{$\tilde \theta_{2}$} & \matindex{$\tilde \theta_{3}$} &  \matindex{$\tilde \theta_{4}$} & \matindex{$ \tilde \theta_{5}$} & \matindex{$\tilde \theta_{6}$} & \matindex{$\tilde \theta_{7}$}\\
        \begin{block}{c@{\hspace{3pt}}c[@{\hspace{3pt}}c@{\hspace{3pt}}c@{\hspace{3pt}}c@{\hspace{3pt}}c@{\hspace{3pt}}c@{\hspace{3pt}}c@{\hspace{3pt}}c@{\hspace{3pt}}c]}
          \matindex{$t_{0}$}& \matindex{$\theta_{0}$}
          & \color{blue}1 & \color{blue} 1  & \color{blue}0& \color{blue}0&\color{blue}0&0&1&  -1\\
          \matindex{$t_{1}$}& \matindex{$\theta_{1}$}
          &  \color{blue}1 & \color{blue}-1 & \color{blue}0 &\color{blue}0&\color{blue}1&1&0&0 \\
          \matindex{$t_{2}$}& \matindex{$\theta_{2}$}
          &  \color{blue}0 & \color{blue}0 & \color{blue}1 &\color{blue}1&\color{blue}1&-1&0&0 \\
          & \matindex{$\theta_{3}$}
          &  0 & 0 & 1 &-1&0&0&1&1 \\
        \end{block}
      \end{blockarray} \Rightarrow   \begin{blockarray}{cccccccccc}
        \ & &\matindex{$\tilde \vartheta_{3}$}  &\matindex{$\tilde \vartheta_{4}$}   &  & &  &\matindex{$\tilde \vartheta_{0}$}   & \matindex{$ \tilde \vartheta_{1}$}& \matindex{$\tilde \vartheta_{2}$}  \\
        \ & &  \matindex{$\tilde \theta_{0}$} & \matindex{$ \tilde \theta_{1}$} & \matindex{$\tilde \theta_{2}$} & \matindex{$\tilde \theta_{3}$} &  \matindex{$\tilde \theta_{4}$} & \matindex{$ \tilde \theta_{5}$} & \matindex{$\tilde \theta_{6}$} & \matindex{$\tilde \theta_{7}$}\\
        \begin{block}{c@{\hspace{3pt}}c[@{\hspace{3pt}}c@{\hspace{3pt}}c@{\hspace{3pt}}c@{\hspace{3pt}}c@{\hspace{3pt}}c@{\hspace{3pt}}c@{\hspace{3pt}}c@{\hspace{3pt}}c]}
          & \matindex{$\theta_{0}$}
          & 1 &1  &0& 0&0& 0&1& -1\\
          \matindex{$t_{0}$}& \matindex{$\theta_{1}$}
          &  \color{blue}1 & \color{blue}-1 &0 &0&1& \color{blue}1&\color{blue}0&\color{blue}0 \\
          \matindex{$t_{1}$}& \matindex{$\theta_{2}$}
          & \color{blue} 0 & \color{blue}0 & 1 &1&1&\color{blue}-1&\color{blue}0&\color{blue}0 \\
          \matindex{$t_{2}$} & \matindex{$\theta_{3}$}
          & \color{blue} 0 &\color{blue}0 &1 &-1&0& \color{blue}0&\color{blue}1&\color{blue}1 \\
        \end{block}
      \end{blockarray}  .
      \\
      \begin{blockarray}{cccccccccc}
        \ & &   & & \matindex{$\tilde \vartheta_{0}$}  &\matindex{$\tilde \vartheta_{1}$} &  \matindex{$\tilde \vartheta_{2}$} & \matindex{$ \tilde \vartheta_{3}$} & \matindex{$ \tilde \vartheta_{4}$}&   \\
        \ & &  \matindex{$\tilde \theta_{0}$} & \matindex{$ \tilde \theta_{1}$} & \matindex{$\tilde \theta_{2}$} & \matindex{$\tilde \theta_{3}$} &  \matindex{$\tilde \theta_{4}$} & \matindex{$ \tilde \theta_{5}$} & \matindex{$\tilde \theta_{6}$} & \matindex{$\tilde \theta_{7}$}\\
        \begin{block}{c@{\hspace{3pt}}c[@{\hspace{3pt}}c@{\hspace{3pt}}c@{\hspace{3pt}}c@{\hspace{3pt}}c@{\hspace{3pt}}c@{\hspace{3pt}}c@{\hspace{3pt}}c@{\hspace{3pt}}c]}
          \matindex{$t_{2}$}& \matindex{$\theta_{4}$}
          & 1 &  1  &  \color{blue} 0& 0&\color{blue}0&\color{blue}0&\color{blue}1& -1\\
          & \matindex{$\theta_{1}$}
          &  1 & -1 & 0 &0&1&1&0&0 \\
          \matindex{$t_{0}$}& \matindex{$\theta_{2}$}
          & 0 & 0 & \color{blue} 1 &1&\color{blue}1&\color{blue}-1&\color{blue}0&0 \\
          \matindex{$t_{1}$}& \matindex{$\theta_{3}$}
          & 0 & 0 &  \color{blue}1 &-1&\color{blue}0&\color{blue}0&\color{blue}1&1 \\
        \end{block}
      \end{blockarray} \Rightarrow  \begin{blockarray}{cccccccccc}
        \ & &  \matindex{$\tilde \vartheta_{1}$} &  \matindex{$\tilde \vartheta_{2}$} & \matindex{$\tilde \vartheta_{3}$} & \matindex{$\tilde \vartheta_{4}$} &  &  & & \matindex{$ \tilde \vartheta_{0}$} \\
        \ & &  \matindex{$\tilde \theta_{0}$} & \matindex{$ \tilde \theta_{1}$} & \matindex{$\tilde \theta_{2}$} & \matindex{$\tilde \theta_{3}$} &  \matindex{$\tilde \theta_{4}$} & \matindex{$ \tilde \theta_{5}$} & \matindex{$\tilde \theta_{6}$} & \matindex{$\tilde \theta_{7}$}\\
        \begin{block}{c@{\hspace{3pt}}c[@{\hspace{3pt}}c@{\hspace{3pt}}c@{\hspace{3pt}}c@{\hspace{3pt}}c@{\hspace{3pt}}c@{\hspace{3pt}}c@{\hspace{3pt}}c@{\hspace{3pt}}c]}
          \matindex{$t_{1}$}& \matindex{$\theta_{4}$}
          & \color{blue}1 &\color{blue}1  & \color{blue}0& \color{blue}0&0&0&1& \color{blue} -1\\
          \matindex{$t_{2}$}& \matindex{$\theta_{5}$}
          & \color{blue} 1 &\color{blue} -1 & \color{blue}0 &\color{blue}0&1&1&0&\color{blue}0 \\
          & \matindex{$\theta_{2}$}
          &0 & 0 & 1 &1&1&-1&0&0 \\
          \matindex{$t_{0}$}& \matindex{$\theta_{3}$}
          &  \color{blue}  0 & \color{blue} 0 & \color{blue}1 &\color{blue}-1&0&  0& \color{blue} 1& \color{blue} 1 \\
        \end{block}
      \end{blockarray} .\end{gathered}
  \end{equation}
\end{figure*}
so that at round $r$ worker $i$ performs task $\theta_{i + 5r \% n'}$ and similarly at round $r$ we have $t_i = \theta_{i+r \% k}$.

More generally we find a displacement $d$ equal to an (odd) prime number that is co-prime\footnote{Although it is notoriously hard to find a prime divisor of number, it is surprisingly  easy to find a prime \textit{non}-divisor. This easy to see since one can just test divisibility by 2,3,5,... and since the product of the first primes less than 100 is approximately equal to $2^{120}$ this will halt \textit{very} quickly, i.e. it will halt in less than 25 steps for $k< 2^{120}$ since there are only 25 primes less than $100$.} to $k$ and we let worker $i$ performs task $\theta_{i + dr \% n'}$  at round $r$ and let $t_i = \theta_{i+r \% k}$ at round $r$. This allows gives the following statistical uniformity lemma:
\begin{lemma}\normalfont\label{lem:spin}
  If the displacement, $d$, is equal to an (odd) prime number that is co-prime to $k$ then the blue rectangle in (the general form of) Eq.~\ref{eq:spin} will visit every entry in the matrix with every possible pattern of $X^{(t)}$ and every cyclic permutation of the $t_i$ contained inside of the blue rectangle.
\end{lemma}
\begin{proof}
  The leftmost point of the blue rectangle is equal to $ (i,i)+r(1,d) \equiv (i + r ,i+dr  ) \text{ mod }\mathbb{Z }/n'\mathbb{Z} \times \mathbb{Z }/k\mathbb{Z} $. By the Chinese remainder theorem (see \cite{ireland1982classical} or \cite{dummit2003abstract}) $(1,d)$ is a generator of $\text{ mod }\mathbb{Z }/n'\mathbb{Z} \times \mathbb{Z }/k\mathbb{Z}$ since $d$ is coprime to $1,$ $k,$ and $n'.$
\end{proof}

\section{Analysis and Evaluation }
\begin{table*}[t]
  \caption{Comparison of main algorithms.}
  \label{tab:compare}
  \vskip 0.15in
  \begin{center}
    \begin{small}
      \begin{sc}
        \begin{tabular}{lcccccc}
          \toprule
          Code   & Encoding          & Communication                 & Decoding                                               & Weight                  & Asynchronous & Parameter    \\
          Scheme & Complexity        & Complexity                    & Complexity                                             & Range                   & ?            & Compression? \\
          \midrule
          LWPD   & 0                 & $\mathcal{O}(\frac{k}{t})   $ & 0                                                      & $t \in [2,\frac{n}{4}]$ & $\surd$      & $\surd$      \\
          GC     & $\mathcal{O}(nk)$ & $\mathcal{O}(k)$              & $\mathcal{O}(k^{\omega}) \leq  \mathcal{O}(k^{2.38}) $ & $t = n-k+1$             & $\times$     & $\times$     \\
          $K$-AC & 0                 & $\mathcal{O}(k)         $     & 0                                                      & $t\in [1,n]$            &
          $\surd $
                 & $\times $                                                                                                                                                          \\
          \bottomrule
        \end{tabular}
      \end{sc}
    \end{small}
  \end{center}
  \vskip -0.1in
\end{table*}
\begin{figure*}
  \begin{subfigure}
    \centering
    \includegraphics[width=.32\textwidth]{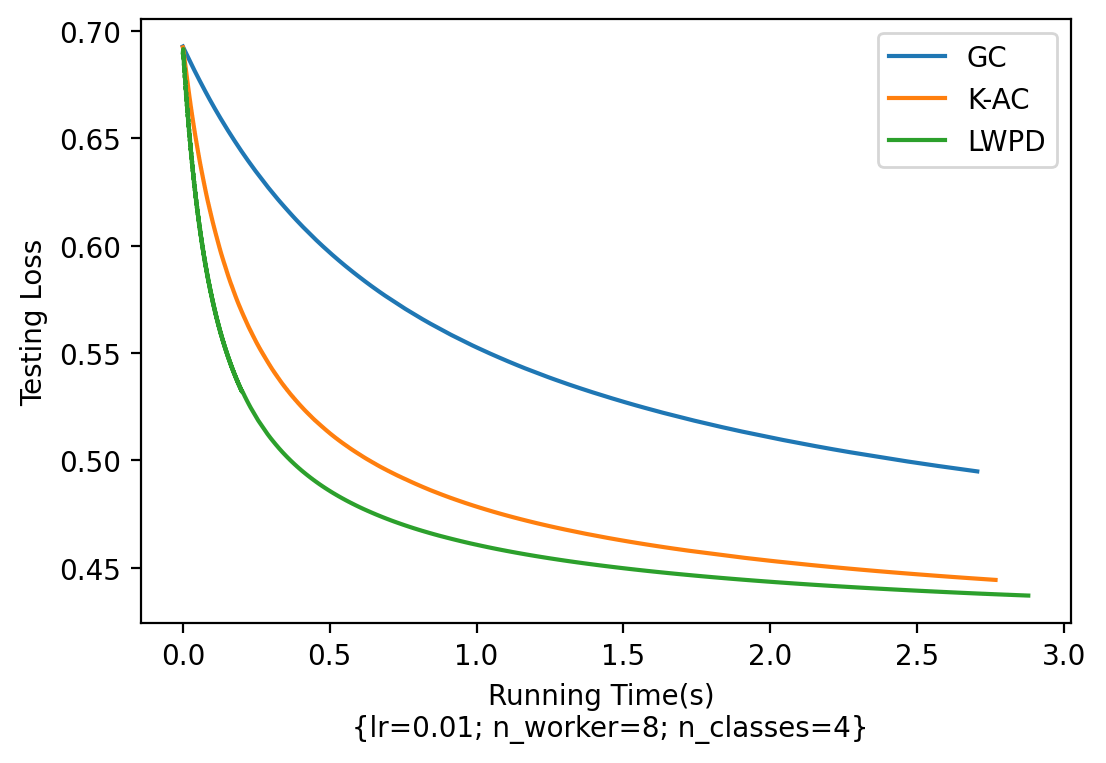}
  \end{subfigure}
  \begin{subfigure}
    \centering
    \includegraphics[width=.32\textwidth]{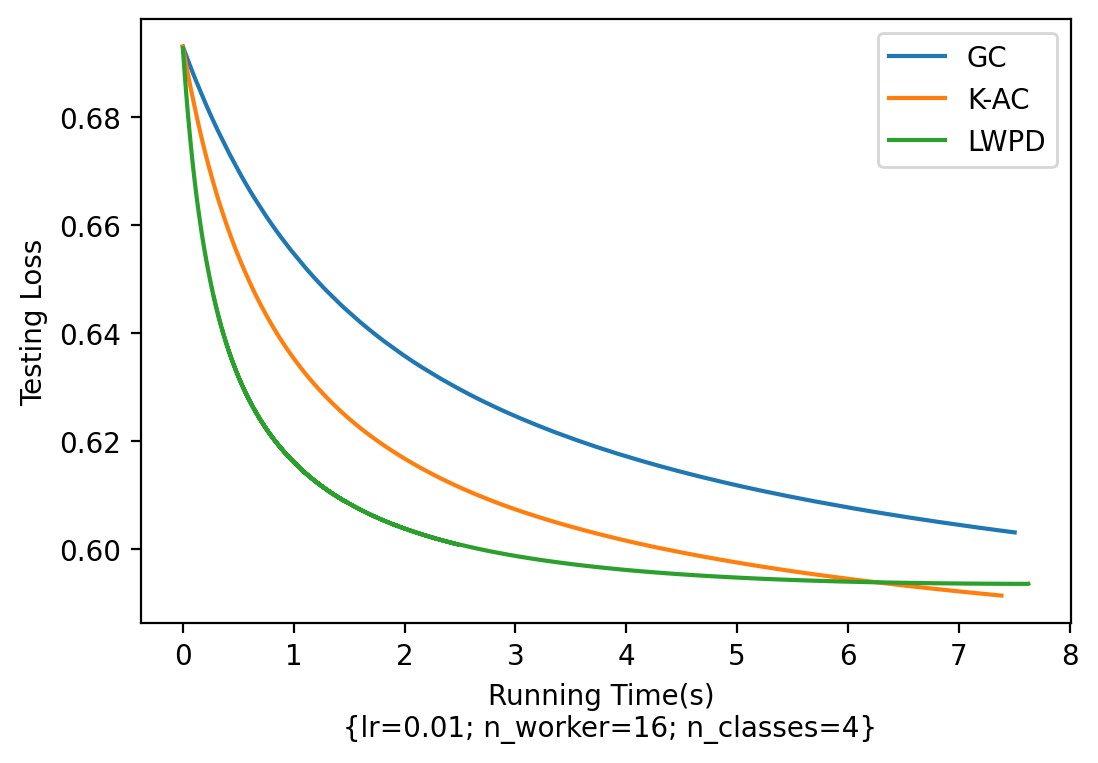}
  \end{subfigure}
  \begin{subfigure}
    \centering
    \includegraphics[width=.32\textwidth]{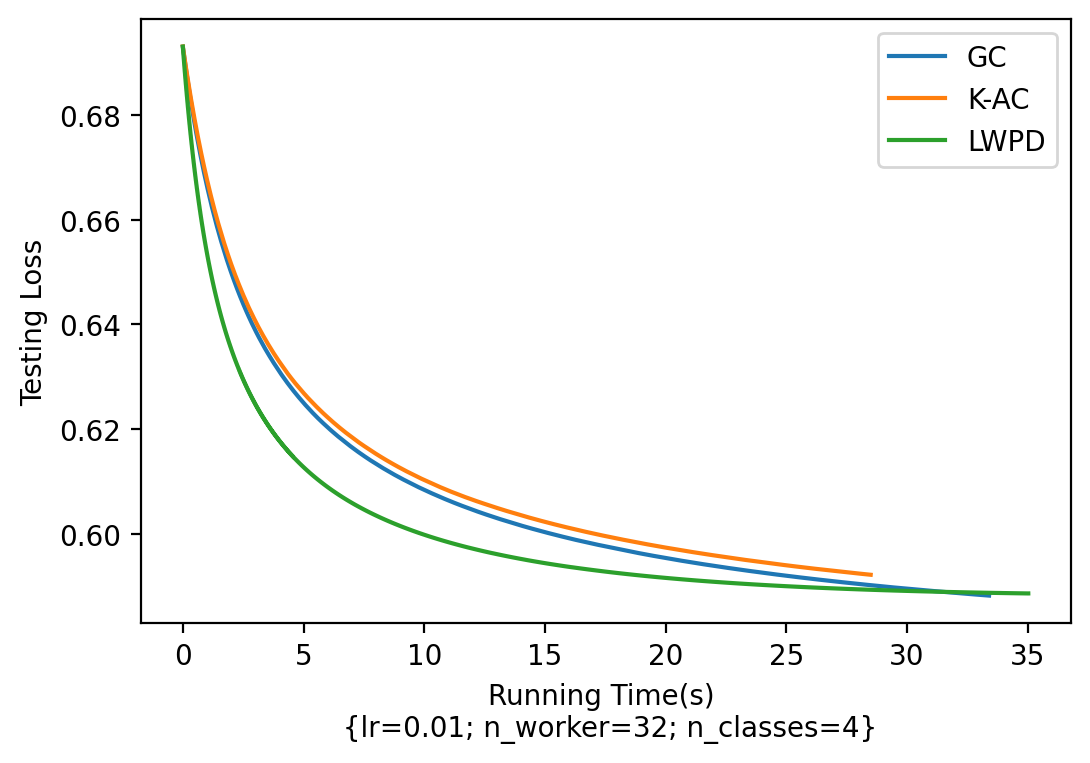}
  \end{subfigure}
  \caption{Experiments with 8, 16, and 32 workers.}
  \label{fig:exp_fig1}
\end{figure*}

\subsection{Theoretical Analysis}
In this section we give a theoretical comparison of the algorithms, see Table~\ref{tab:compare}, and we prove theorems regarding the existence and non-existence of codes with certain properties. The following theorem, \emph{i.e.,} Thm.~\ref{thm:no_mds}, shows that Hamming-distance MDS coding schemes must have the workers do an arbitrarily large amount of work. We then later show that our codes are approximately MDS with respects to the projective geometry metric which maximize the amount of information sent back by the workers\footnote{This is because large angles gives us large conditional entropy.} while keeping the amount of work done by the workers as low as possible; \emph{i.e.,} there are approximately projective-MDS that have weights $t=2,...,n$.
\begin{theorem}\label{thm:no_mds}\normalfont
  If the parameters $(n,k,t)$ satisfy
  $t \leq n-k$ then there is no Hamming-distance MDS $(n,k)$-code for the derivatives.
\end{theorem}
\begin{proof}
  If $A(\mathcal{C})_i=$``number of rows of weight $i$'', then Theorem 7.4.1 in \cite{2003fundamentals} gives us that an MDS will have $A(\mathcal{C})_i=0$ for $i \leq n-k$.
\end{proof}
In particular; the proof of Thm.~\ref{thm:no_mds} can be strengthened to say that:
\begin{corollary}\label{cor:must_work}  In an MDS $(n,k)$-coding scheme $A(\mathcal{C})_i=0$, for $i \leq n-k$, where $A(\mathcal{C})_i=0$ is the weight distribution of a code.
\end{corollary}
The importance of Cor.~\ref{cor:must_work} is made clear through the following interpretation:
\begin{corollary}
  In an MDS $(n,k)$-coding scheme all of the workers must do at least $n-k$ amount of work.
\end{corollary}

However a simple observation of the construction given in Sec.~\ref{subsec:pow_two_con} gives us that:
\begin{theorem}\normalfont
  There exists $(n,k,t)$-LWPD codes for any $t \geq 2$.
\end{theorem}

The next theorem proves that under the projective distance we have that our code achieves approximately maximal distance.
\begin{theorem}\normalfont\label{thm:opt}
  The family $(n,k,t)$-code are approximately MDS $(n,k)$-code for the derivatives in the  projective-distance for $n \leq 2k$.
\end{theorem}
\begin{proof}
  By Lem.~\ref{lem:spin} it suffices to prove this for powers of two.
  The distance between any vectors is $\arccos \frac{1}{2}= \frac{\pi}{3}$ and this only happens for $t$ out of $n$ choices for any vector; the distance is equal to the maximum $\frac{\pi}{2}$ for all other vectors.
\end{proof}

The proof of Thm.~\ref{thm:opt} gives us that the distance between any two codewords $\tilde \theta, \tilde \theta$ is bounded above by
$
  d(\mathcal{C}) = \min_{\tilde \theta, \tilde \theta \in \mathcal{C}} d(  \tilde \theta, \tilde \theta) = \frac{\pi}{3}
$
and thus in term of percentages of the optimal $\frac{\pi}{2}$ we have
\begin{equation}\label{eq:per_mds}
  \frac{ \frac{\pi}{2} - d(\mathcal{C}) }{\frac{\pi}{2} } = \frac{1}{6} \approx 16 \%
\end{equation}
of the ``theoretical'' optimal distance; however there can be no code that achieves the ``theoretical'' optimal distance:
\begin{theorem}
  The percentage in Eq.~\ref{eq:per_mds} cannot be made 100\%; \emph{i.e.,} there are no projective MDS codes for $n>k$ which achieve distance $\frac{\pi}{2}$.
\end{theorem}
\begin{proof}
  If this statement was false there would be $n+1$ linearly independent vectors in $n$-dimensional space, $\mathbb{F}^n$.
\end{proof}

An interesting fact about the bound given by Thm.~\ref{thm:opt} is that is a constant independent of the dimension of the code and thus it scales well for larger and larger number of workers.

%

\subsection{Experimental Results}

The experiments were run on AWS Spot Instances; the workers are AWS EC2 c5a.large instances (compute optimized) and master is an AWS EC2 r3.large instance (Memory Optimized).
The experimental procedure was written using Mpi4py \cite{d1, d2} in Python. We used a modification of the code in \cite{Rashish2017} written by the first author of \cite{pmlr-v70-tandon17a} to implement Gradient Coding (GC) as well as the random data generation; the implementation only supported logistic regression and we generalized it to support multinomial logistic regression (\emph{i.e.,} more than one class). The software in \cite{Rashish2017} used a Gaussian mixture model of two distributions to create input features for the logistic model; we generalized it to allow for an arbitrary number of Gaussian distributions in the mixture to create a robust data set.
To be as fair as possible in our comparison with $K$-Asynchronous Gradient descent ($K$-AC) \cite{pmlr-v84-dutta18a} we made setup for $K$-AC nearly identical with the exception of the coding scheme; \emph{i.e.,} $K$-AC and LWPD used the exact same data partitions and same number of $k$ workers in the $k$-asynchronous batches.

We ran experiments with 8 workers, 16 workers, and 32 workers (see Fig.~\ref{fig:exp_fig1}). The testing error (\emph{i.e.,} the workers do not train on the test data) is plotted against the time.
In all of the experiments we ran LWPD codes converged far faster; however, it often overfitted and sometimes the other algorithms would eventually get a lower test error.
There are two possible explanations for this: either LWPD converges so much faster than the other algorithms that they never get a chance to overfit or the noise that initially helps LWPD find a very quick solution eventually causes it to stay some distance from the optimal solution.
There is evidence for both of these possibilities because there are experiments were the other algorithms do not catch up to LWPD; see Sec.~\ref{app:exp} for more results.

\begin{figure*}
  \begin{subfigure}
    \centering
    \includegraphics[width=.32\textwidth]{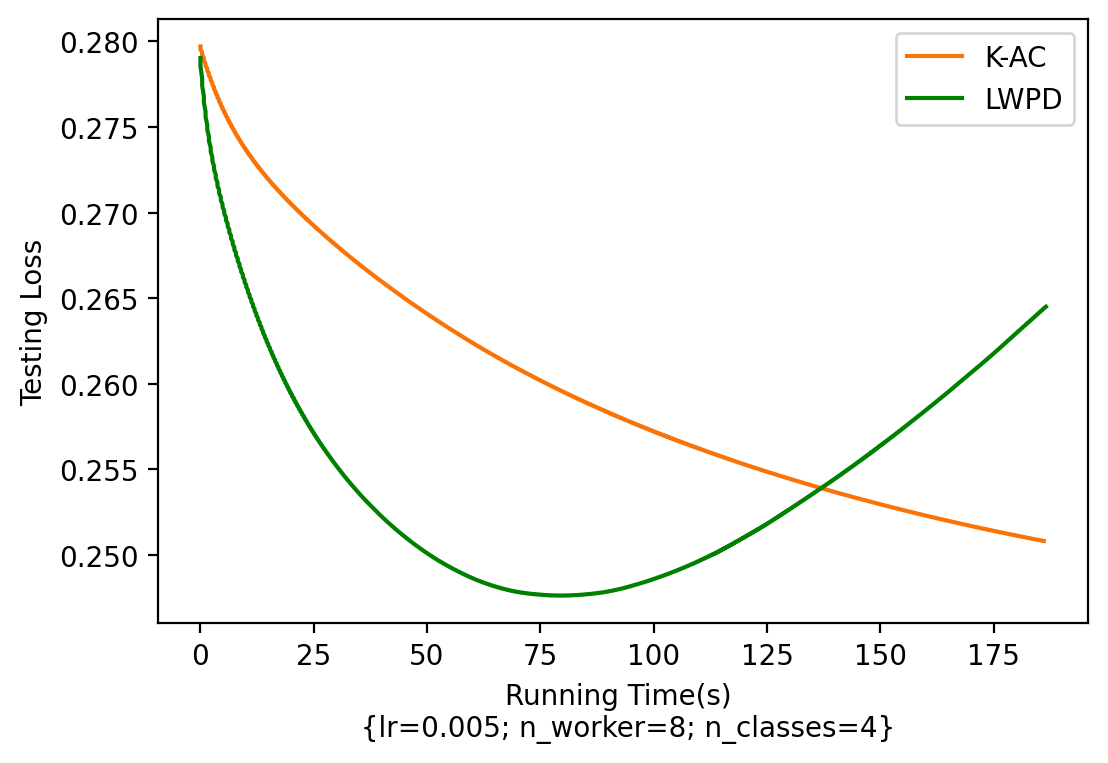}
  \end{subfigure}
  \begin{subfigure}
    \centering
    \includegraphics[width=.32\textwidth]{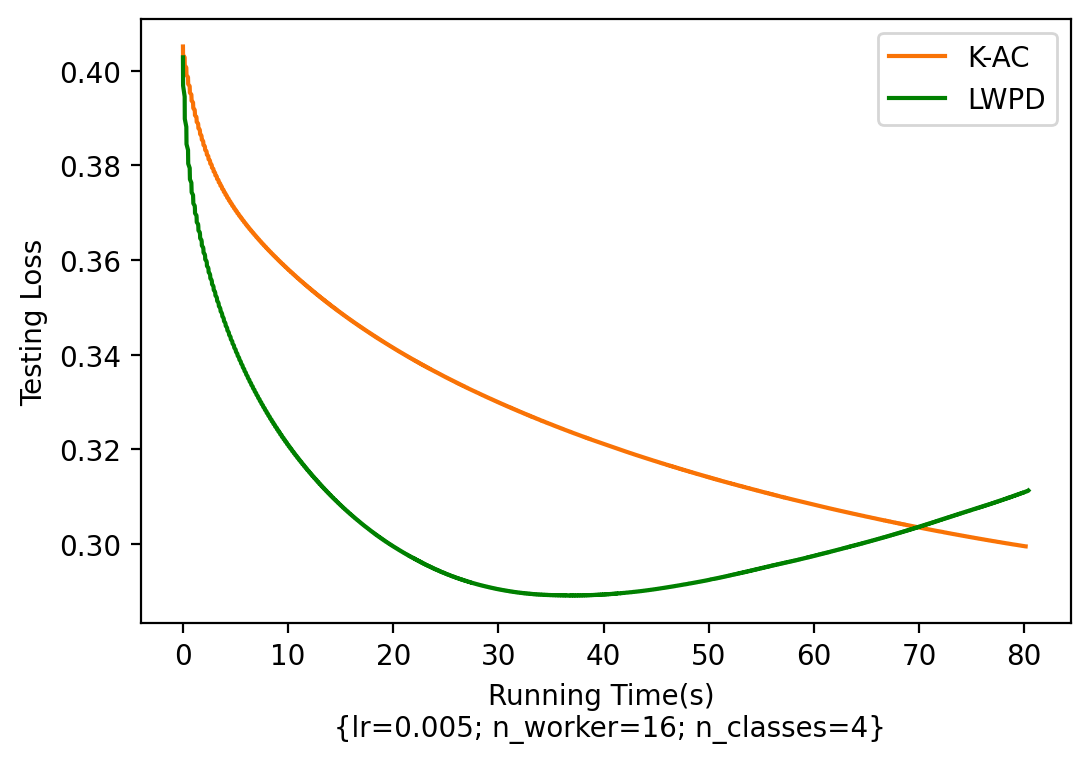}
  \end{subfigure}
  \begin{subfigure}
    \centering
    \includegraphics[width=.32\textwidth]{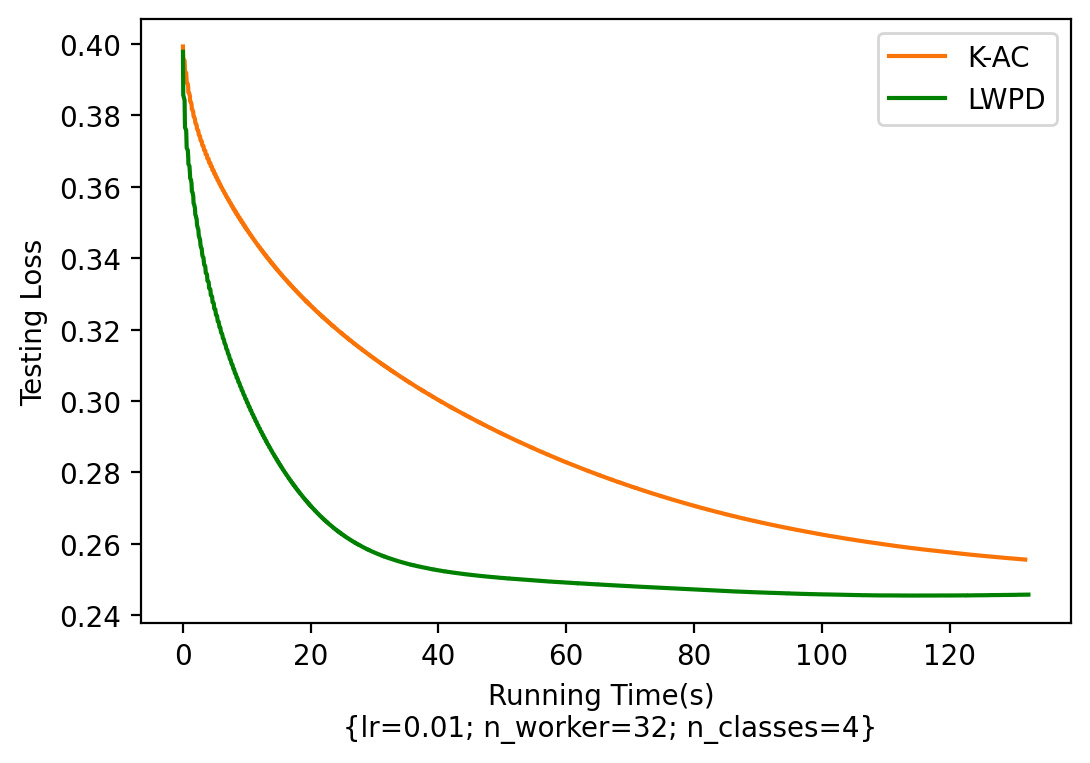}
  \end{subfigure}
  \caption{Further experiments with 8, 16, and 32 workers.}
  \label{fig:exp_fig2}
\end{figure*}
We have modified and re-ran our experiments to allow for deep neural networks (see Fig.~\ref{fig:exp_fig2}).
The new architecture had ReLU activations at the hidden nodes, Softmax at the output nodes, and cross entropy as the loss function.
The experiments were run using Mpi4py on AWS Spot Instances with AWS EC2 c5a.large for the workers and AWS EC2 r3.large for the master. 
The original authors' implementation of Gradient Coding did not support deep neural networks and thus we simply compared with $K$-Asynchronous Gradient descent for these experiments.
Both the updated and original experiments contained stragglers and highlighted the asynchronous nature of the algorithm, since we used Amazon Spot instances, which are the most unreliable but cost-efficient machines that Amazon AWS can offer.

In our experimental setup, all workers compute training loss individually. These plots should not be considered a centralized measure of model performance because the data set is \emph{de}centralized. Thus, plotting training metrics does not make as much sense as, for example, visualizing the testing/validation set error or loss. Furthermore, we were guided by the intuition that the validation error is the more important measure, since it measures how well the algorithm or model can perform on data it has not seen.

The experimental data plots show \emph{test loss} ($y$-axis) versus the \emph{time elapsed} ($x$-axis).
The test loss refers to the error on the validation set (which the model is not trained on) and thus overfitting occurs when the training loss continues to decrease, but the generalization error stops decreasing and eventually begins to grow.
In our experiments, we observed that other methods reach similar convergence point much later in terms of number of iterations. One hypothesis is that our gradient coding approach offers faster training. A second possible explanation comes from the lossy compression of our algorithm. We trade off error for runtime. Our hypothesis is that the noise coming from the trade-off is beneficial in that it prevents us from being stuck in any potential saddle point.

\section{Conclusion and Open Problems}
We propose LWPD codes which allow for asynchronous gradient updates by maximizing the amount of information contained by random subsets of vectors and minimizing the weight of the code.
Our code compresses the gradient in a manner that scales well with the number of nodes (and the dimension of the data) and achieves a lower a communication complexity and memory overhead with respect to the state of the art.
Another improvement of our algorithm over the state of the art is our discovery of the correct information metric; all of the other coding schemes assume that the Hamming distance is the correct metric, which does not consider the natural (differential) geometry of the gradient.
Furthermore, we showed that our code was very efficient since the master can just directly add and subtract the results returned by the workers without needing to decode the information.
We proved many of the complexity guarantees theoretically and also provided much empirical evidence for the performance.
For future work we would like to strengthen the theoretical results by proving stronger complexity bounds as well as further investigating the effect of noise (or lossy compression) on the performance; it seems that at first the lossy compression is a great help but eventually it causes over fitting.
\section*{Acknowledgments}
This paper is based upon work supported by the National Science
Foundation under Grant No. CCF-2101388.
\bibliography{references0, ref}
\bibliographystyle{icml2022}

\newpage
\appendix
\onecolumn

\section{Experimental Results}\label{app:exp}
\begin{figure}[H]
  \begin{subfigure}
    \centering
    \includegraphics[width=.44\textwidth]{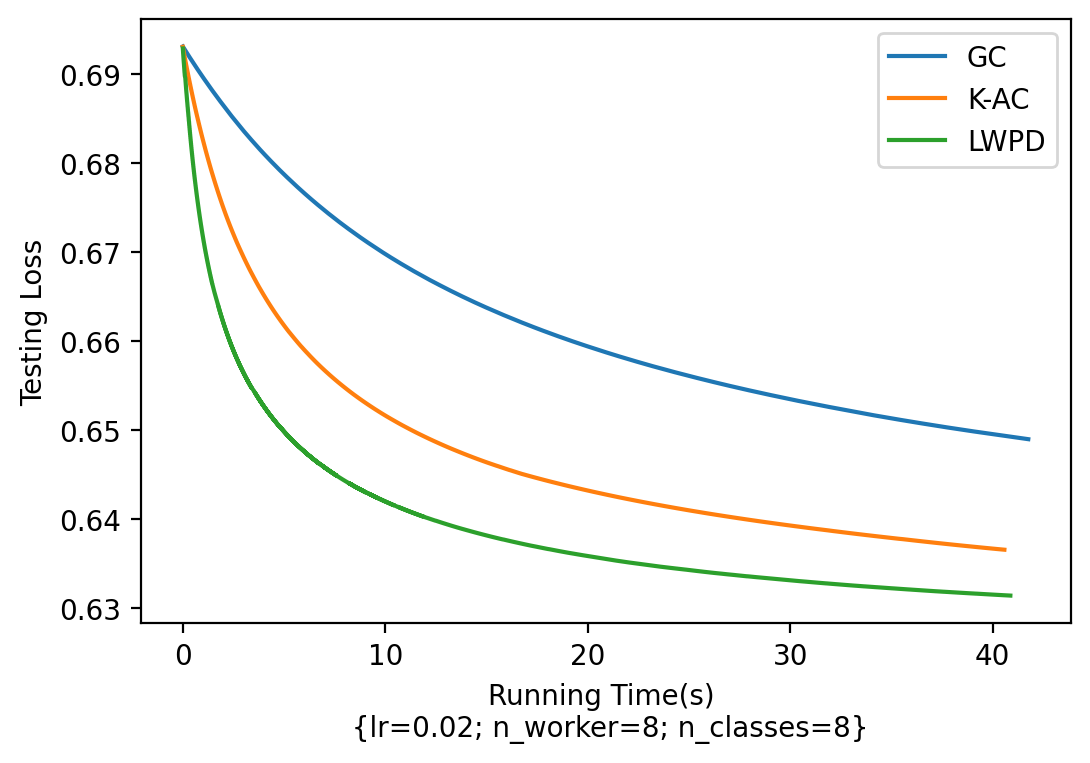}
  \end{subfigure}
  \begin{subfigure}
    \centering
    \includegraphics[width=.44\textwidth]{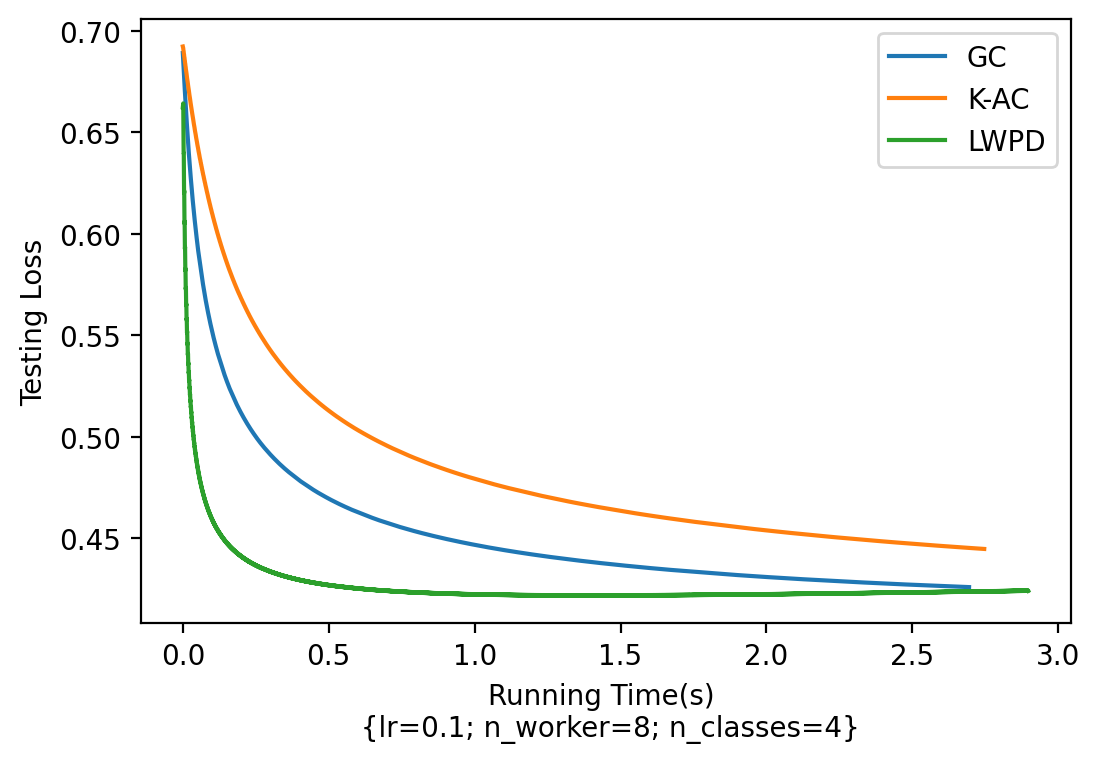}
  \end{subfigure}
  \begin{subfigure}
    \centering
    \includegraphics[width=.44\textwidth]{ICML_testingloss_n9_lr_001_nout_4.png}
  \end{subfigure}
  \begin{subfigure}
    \centering
    \includegraphics[width=.44\textwidth]{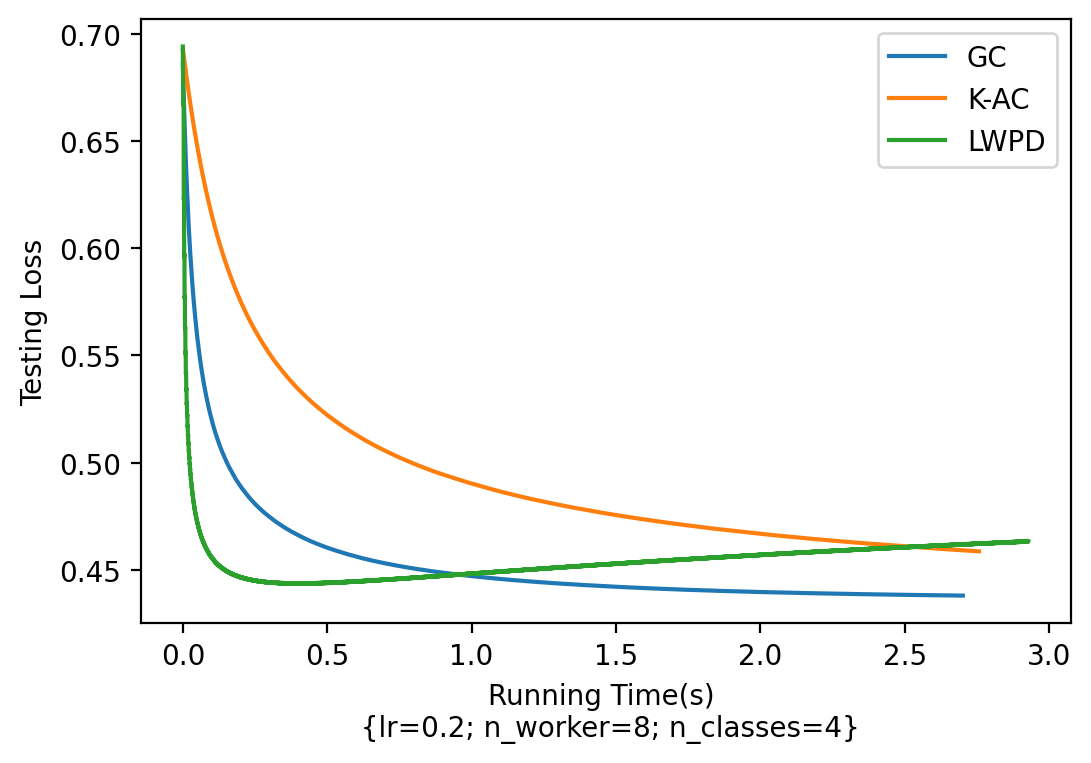}
  \end{subfigure}
  \begin{subfigure}
    \centering
    \includegraphics[width=.44\textwidth]{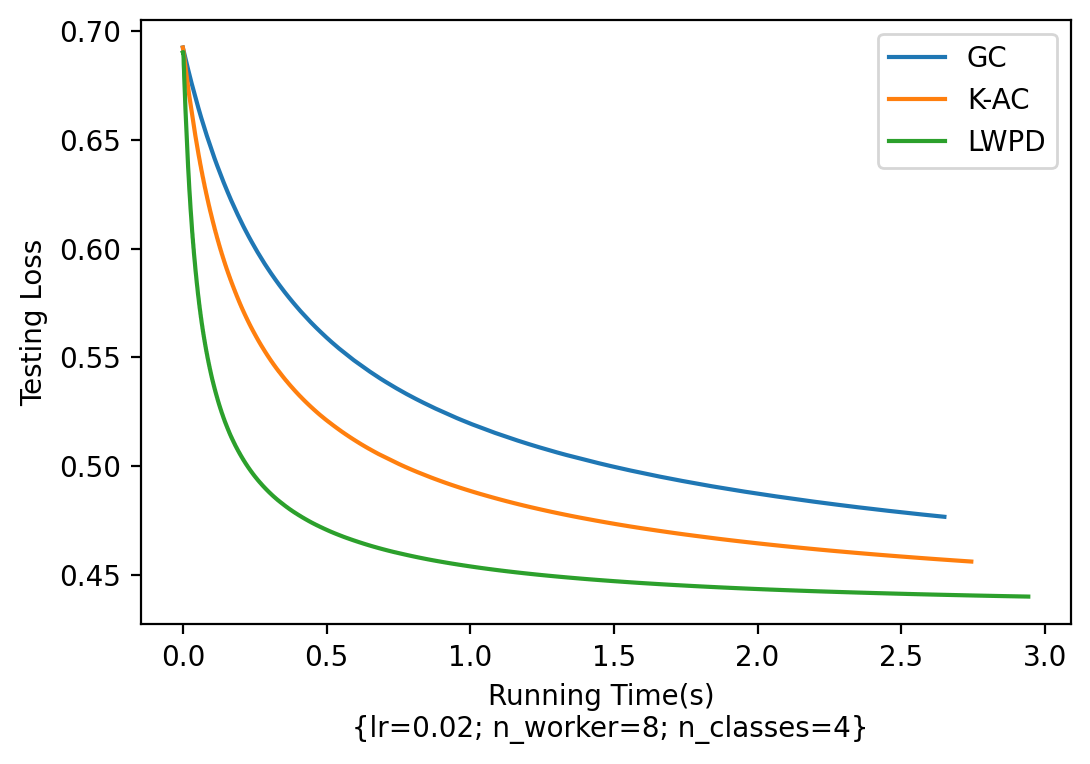}
  \end{subfigure}
  \hfill
  \begin{subfigure}
    \centering
    \includegraphics[width=.44\textwidth]{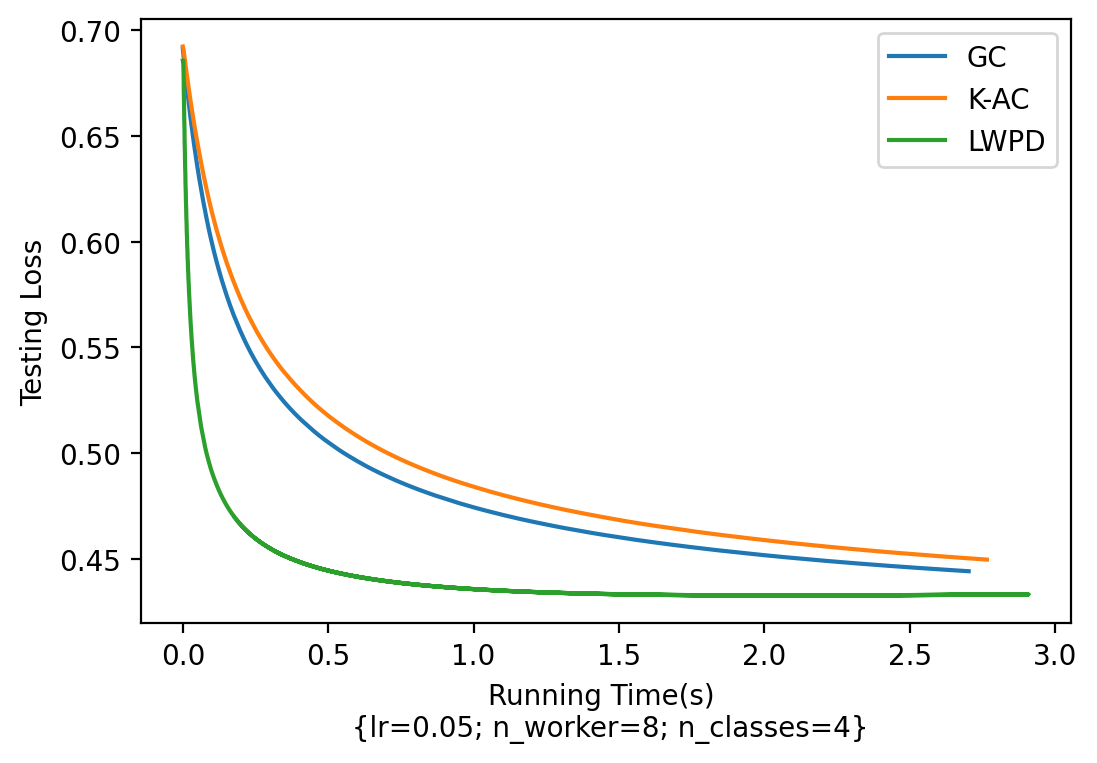}
  \end{subfigure}

  \caption{Experiments with 8 workers.}
\end{figure}

\newpage
\begin{figure}[H]
  \begin{subfigure}
    \centering
    \includegraphics[width=.44\textwidth]{ICML_testingloss_n17_lr_001_nout_4.png}
  \end{subfigure}
  \begin{subfigure}
    \centering
    \includegraphics[width=.44\textwidth]{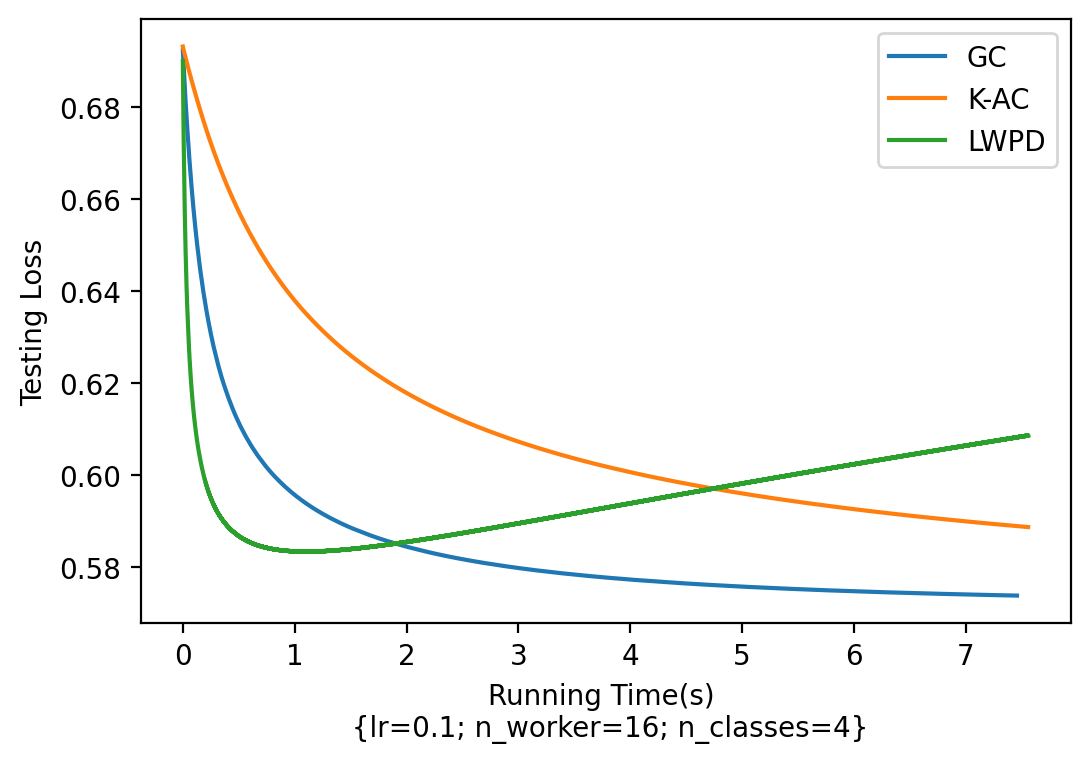}
  \end{subfigure}
  \hfill
  \begin{subfigure}
    \centering
    \includegraphics[width=.44\textwidth]{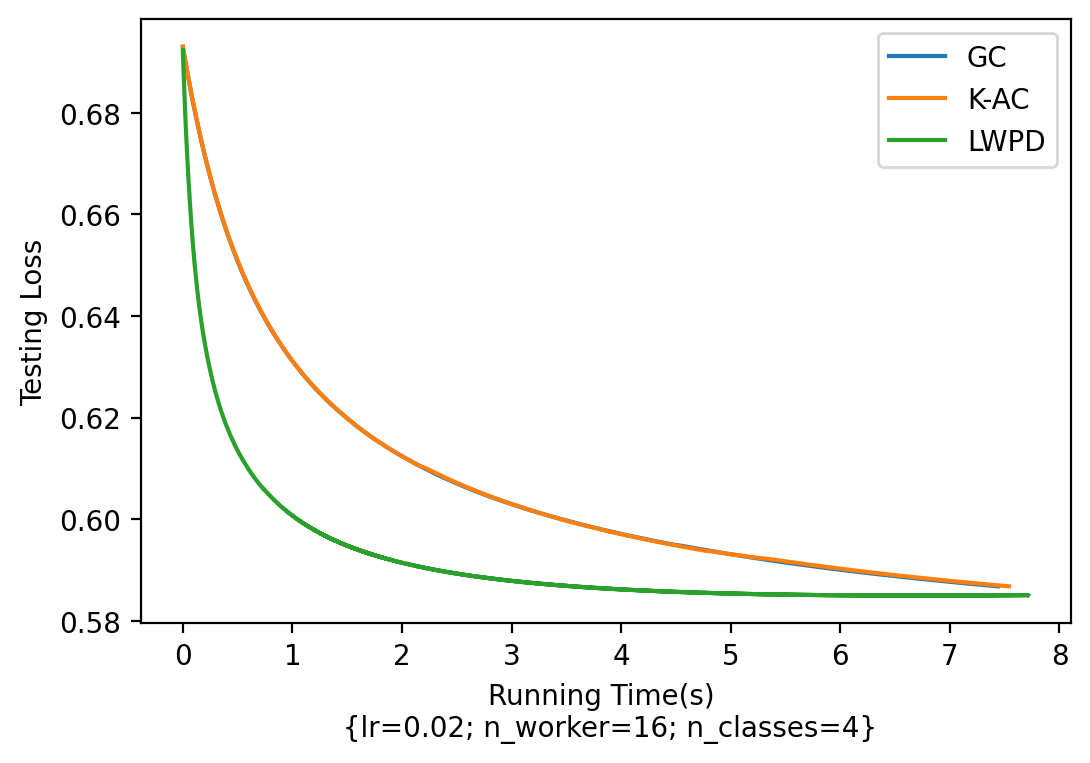}
  \end{subfigure}
  \hfill
  \begin{subfigure}
    \centering
    \includegraphics[width=.44\textwidth]{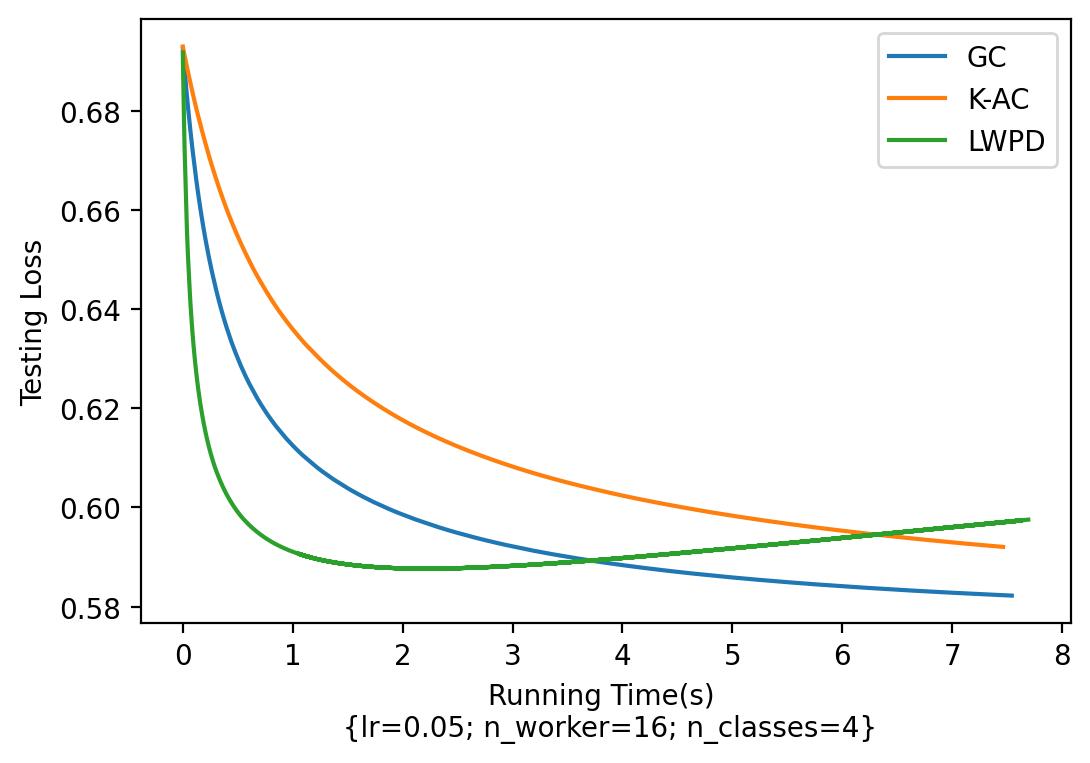}
  \end{subfigure}

  \caption{Experiments with 16 workers.}
\end{figure}

\newpage
\begin{figure}[H]
  \begin{subfigure}
    \centering
    \includegraphics[width=.44\textwidth]{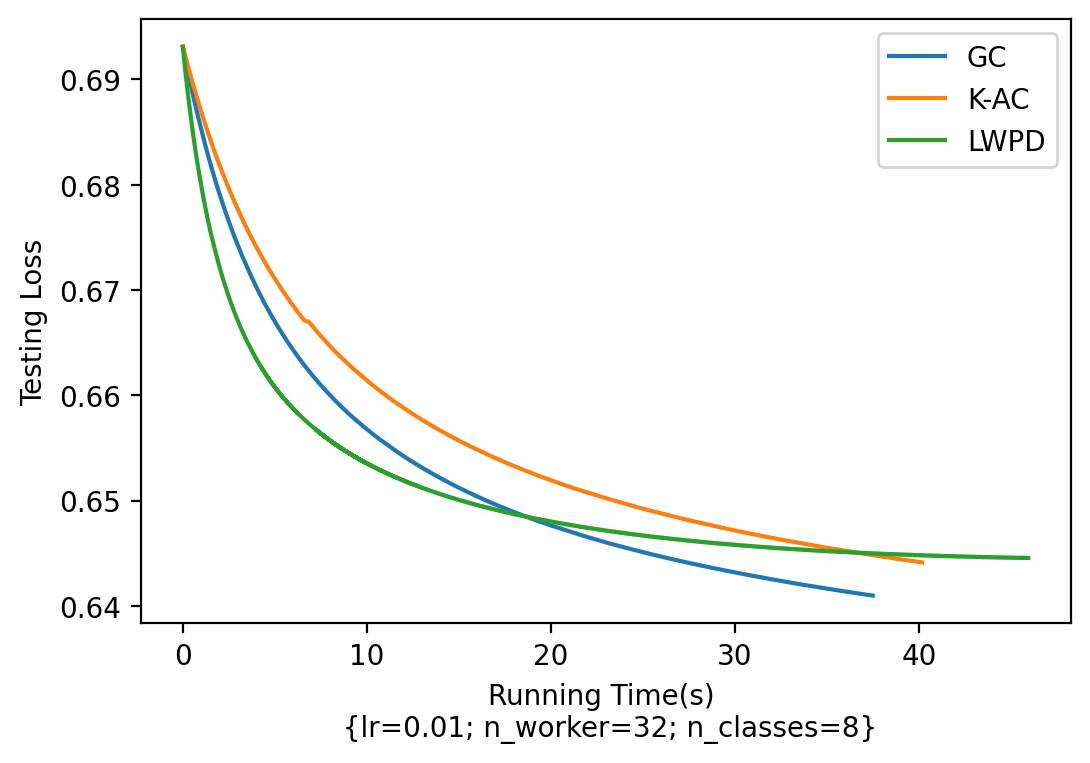}
  \end{subfigure}
  \begin{subfigure}
    \centering
    \includegraphics[width=.44\textwidth]{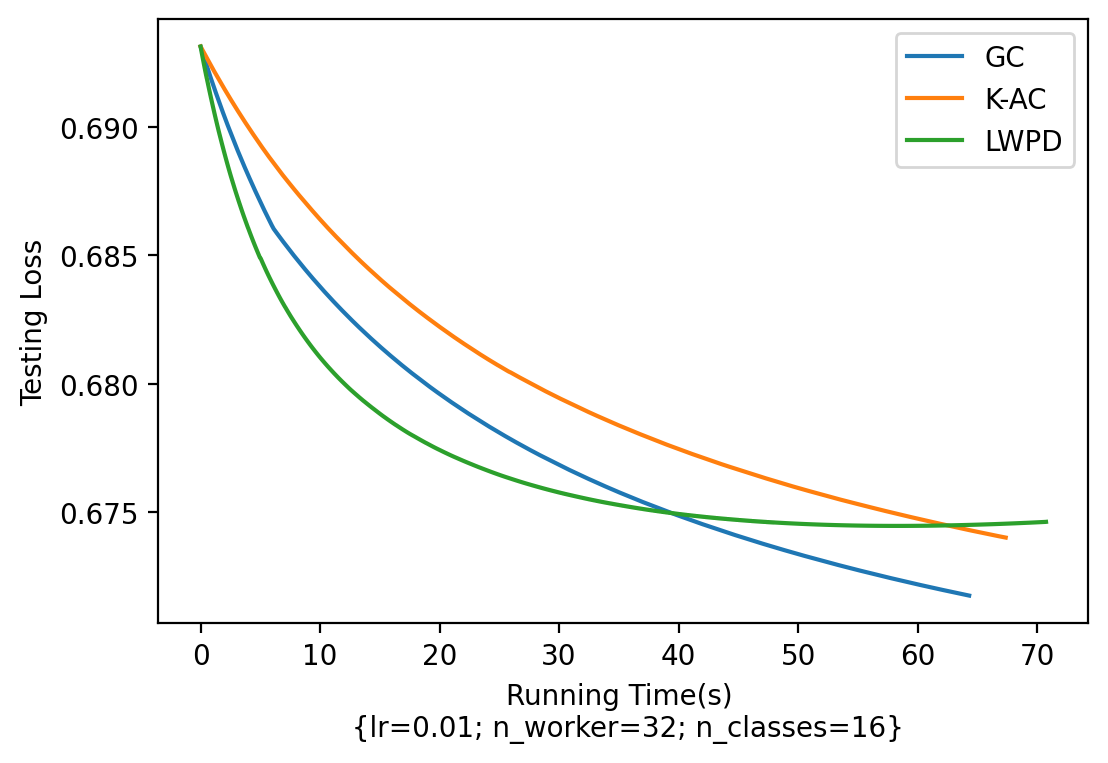}
  \end{subfigure}
  \begin{subfigure}
    \centering
    \includegraphics[width=.44\textwidth]{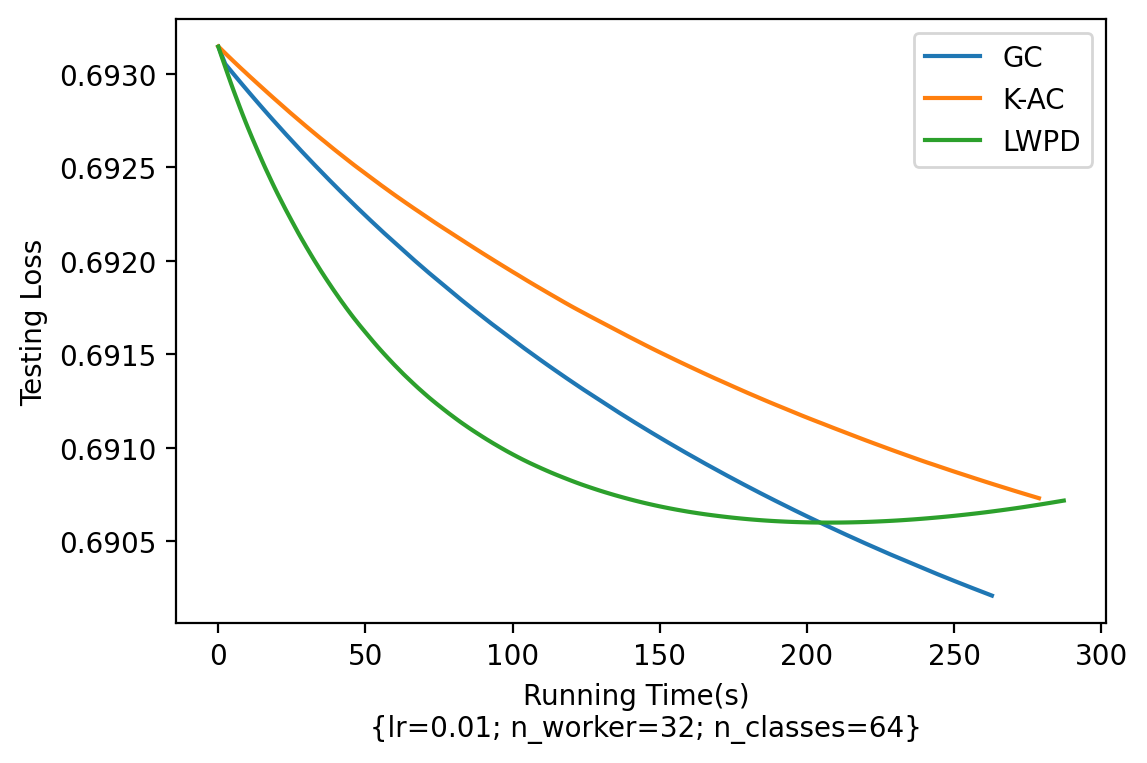}
  \end{subfigure}
  \begin{subfigure}
    \centering
    \includegraphics[width=.44\textwidth]{ICML_testingloss_n33_lr_001_nout_4_npart_4.png}
  \end{subfigure}
  \begin{subfigure}
    \centering
    \includegraphics[width=.44\textwidth]{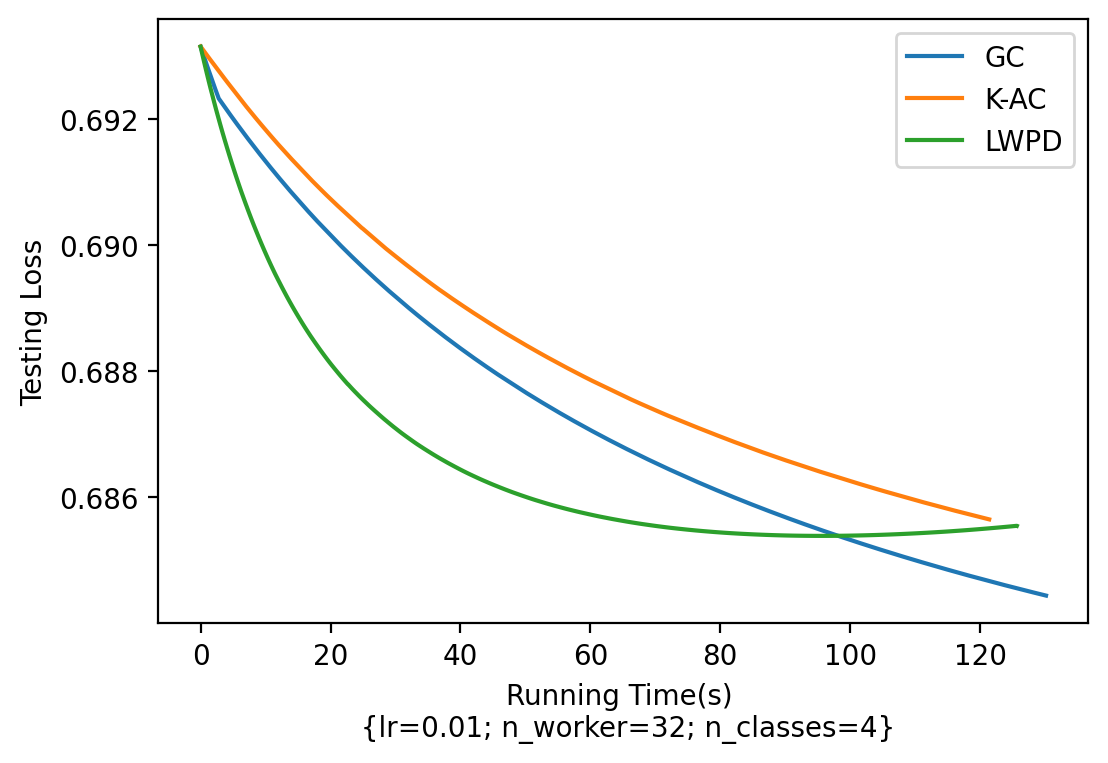}
  \end{subfigure}
  \caption{Experiments with 32 workers.}
\end{figure}
%

\end{document}